\documentclass[11pt]{article}

\usepackage{acl}
\usepackage{times}
\usepackage{latexsym}
\usepackage[T1]{fontenc}
\usepackage[utf8]{inputenc}
\usepackage{microtype}
\usepackage{inconsolata}
\usepackage{graphicx}
\usepackage[most]{tcolorbox}     
\usepackage[dvipsnames]{xcolor}  

\newtcolorbox[list inside=prompt,auto counter,number within=section]{prompt}[1][]{
    colbacktitle=black!60,
    coltitle=white,
    fontupper=\footnotesize,
    boxsep=5pt,
    enhanced,
    left=0pt,
    right=0pt,
    top=0pt,
    bottom=0pt,
    boxrule=1pt,
    breakable,
    #1
}
\usepackage{amsmath,amssymb,amsthm}
\usepackage{mathtools}
\usepackage{bm}
\usepackage{bbm}
\usepackage{booktabs}
\usepackage{multirow}
\usepackage{threeparttable}
\usepackage{tabularx}
\usepackage{xcolor}
\usepackage{colortbl}
\usepackage{enumitem}
\usepackage{subcaption}
\usepackage{titletoc}
\usepackage{algorithm}
\usepackage{algpseudocode}
\usepackage{algorithmicx}

\definecolor{tigerblue}{RGB}{218,232,252}

\newtheorem{theorem}{Theorem}[section]
\newtheorem{proposition}[theorem]{Proposition}

\newtheorem{assumption}[theorem]{Assumption}

\theoremstyle{remark}

\theoremstyle{definition}

\newcommand{\method}{TIGER}

\title{Self-Correction Can Amplify Hallucinations: Fact-Level Repair with Graph-Based Evidence Routing in Multimodal Generation}
\author{
  \normalfont\normalsize
  Kaixiang Zhao\textsuperscript{1} \quad
  Tianrun Yu\textsuperscript{1} \quad
  Shawn Huang\textsuperscript{1} \\[2pt]
  \normalfont\normalsize
  Porter Jenkins\textsuperscript{1} \quad
  Yushun Dong\textsuperscript{2,*} \quad
  Amanda Hughes\textsuperscript{1,*} \\[6pt]
  \textsuperscript{1}Brigham Young University \quad
  \textsuperscript{2}Florida State University \\[4pt]
  \small\texttt{\{kzhao2@byu.edu, yushun.dong@fsu.edu, amanda\_hughes@byu.edu\}} \\[2pt]
  \small\textsuperscript{*}Corresponding authors
}

\begin{document}
\maketitle

\begin{abstract}
We study fact-level repair for multimodal generation, where a fluent output may contain specific facts that are not supported by the input. Existing inference-time repair methods often generate feedback by jointly conditioning on the input and the current output. This design has two limitations: hallucinated claims in the output can bias the model's interpretation of the input, and free-form feedback cannot be ranked or scheduled at the fact level. We present \method{}\footnote{Code released at \url{https://github.com/kzhao5/tiger}}, an inference-time framework that redesigns feedback for localized repair. \method{} independently extracts an observation graph from the input and a claim graph from the current output, then assigns each claim a graph-conditioned risk score based on support and conflict. The model repairs selected high-risk claims while keeping the backbone frozen. We provide a convergence analysis showing that the expected total risk decreases geometrically to an explicit asymptotic bound under mild assumptions. Experiments across four cross-modal paths, including image$\to$text, image+text$\to$text, audio$\to$text, and video$\to$text, show that \method{} reduces unsupported content while preserving task quality. The gains hold across multiple backbones, and a CrisisFACTS case study suggests that the same repair mechanism can improve grounding in multi-source settings.

\end{abstract}

\section{Introduction}
\label{sec:introduction}

Recent unified multimodal language models support cross-modal generation with inputs from text, images, audio, and video~\citep{wu2023next,lu2023unified,chameleon2024,anygpt2024,janus2024,januspro2025,qwen25omni2025,bagel2025,mmada2025}. This capability is useful for applications such as content creation, multimodal decision support, disaster response, medical triage, and journalism. In this paper, we focus on factual generation with textual outputs, where the input may contain one or more modalities. A key challenge in this setting is \emph{faithfulness hallucination}, where the generated response contains facts that are not supported by the input~\citep{crossmodalconsistency2024,samecontent2025}. The response is often fluent, and most of its content may be correct, but a small number of specific facts can contradict the evidence. For example, an image with two trucks may be described as containing three trucks, an intact bridge may be described as collapsed, or ``light rain'' in speech may be summarized as a ``heavy storm.'' These errors are local, but they can be harmful because downstream users may trust the generated content. Section~\ref{sec:motivation} shows that such errors can arise from spurious correlations during joint autoregressive generation. \emph{Our goal is to identify and repair unsupported facts at inference time on a frozen backbone, while preserving correct content.}

\begin{figure*}[t]
\centering
\includegraphics[width=\textwidth]{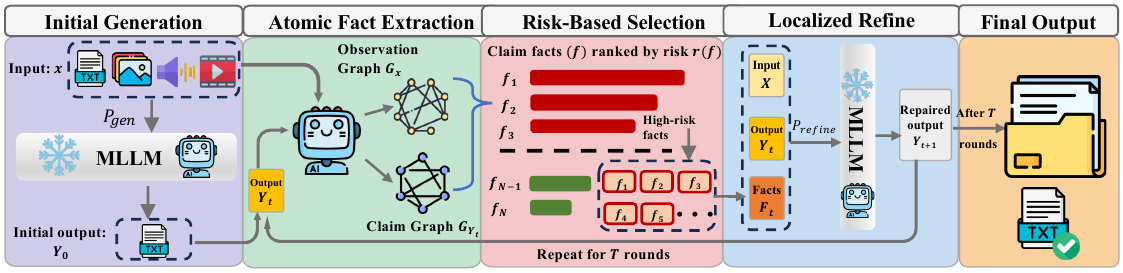}
\caption{\textbf{Overview of TIGER.} TIGER first generates an initial output, extracts fact graphs from the input and output, ranks claims by risk, and locally repairs selected high-risk facts while keeping the backbone frozen.}
\label{fig:tiger_pipeline}
\end{figure*}

Iterative repair with feedback is a natural way to address this problem. In this setting, the model revises its current output by conditioning on the input, the current output, and a feedback signal. Existing multimodal self-correction methods, including Volcano~\citep{lee2024volcano}, Woodpecker~\citep{yin2024woodpecker}, and DeGF~\citep{degf2025}, follow this pattern; see Appendix~\ref{sec:related}. However, these methods generate feedback by jointly conditioning on the input and the current output. This design creates a failure mode: the feedback generator sees the hallucinated claims in the current output, so these claims can influence how the model interprets the input. As a result, the feedback may endorse unsupported content instead of flagging it~\citep{sun2026friendly,li2023evaluating,chen2026learning,yu2026provable}. Free-form feedback also provides no explicit score for each claim, which makes it difficult to decide which facts should be repaired under a limited compute budget. These limitations motivate our question:

\begin{center}
\emph{Can we redesign the feedback mechanism in iterative multimodal repair so that it reduces spurious correlation during feedback generation and supports fact-level scheduling?}
\end{center}

We answer this question with \textbf{TIGER}
(Fac\textbf{t}-level repa\textbf{i}r with \textbf{G}raph-based
\textbf{E}vidence \textbf{R}outing), an inference-time framework
that operates on a frozen backbone.
As shown in Figure~\ref{fig:tiger_pipeline}, TIGER redesigns feedback in two stages. First, it applies \emph{atomic projection}: the input and the current output are extracted independently into structured fact graphs. This separation reduces the direct channel through which unsupported claims in the output can influence the model's reading of the input. Second, TIGER computes a \emph{deterministic fact-level risk} for each output claim by measuring its support and conflict against the input graph. These scores make the feedback rankable, so TIGER can select a small set of high-risk facts and repair them locally. We also provide a convergence analysis which shows that the expected total risk decreases geometrically under mild assumptions.

We evaluate TIGER on four text-output cross-modal generation paths: image$\to$text, image+text$\to$text, audio$\to$text, and video$\to$text. The main experiments use Qwen2.5-Omni~\citep{qwen25omni2025} as the primary frozen backbone, with additional results on LLaVA-1.5 and other open-source and proprietary backbones. Across COCO, AMBER, MMHal-Bench, Clotho, and VideoHallucer, TIGER reduces unsupported claims while preserving or improving task quality. On external hallucination metrics such as CHAIR~\citep{rohrbach2018}, TIGER mitigates the common failure mode where repair improves fluency but introduces new unsupported objects. A case study on CrisisFACTS~\citep{crisisfacts2022} shows that TIGER can improve grounded situation reporting from noisy evidence collected from multiple sources.

\noindent\textbf{Our contributions} are summarized as follows:
\begin{itemize}[leftmargin=*]
    \item We identify a feedback-stage failure mode in iterative multimodal repair: when feedback is generated by jointly conditioning on the input and the current output, hallucinated claims in the output can bias the model's interpretation of the input.

    \item We propose \method{}, an inference-time repair framework that replaces free-form joint feedback with independent atomic projection and deterministic fact-level risk ranking. This design makes feedback explicit, rankable, and suitable for localized repair under a fixed compute budget.

    \item We evaluate \method{} across four text-output cross-modal paths: image$\to$text, image+text$\to$text, audio$\to$text, and video$\to$text. Experiments show that \method{} reduces hallucination while preserving task quality, generalizes across multiple backbones, and remains effective in a real-world crisis-reporting case study.
\end{itemize}



\section{A Motivating Observation}
\label{sec:motivation}

Figure~\ref{fig:spurious}~(top) illustrates a simple failure mode.
The input image shows a person wakeboarding on water, but no boat is visible.
With a free-form description prompt, Qwen2.5-Omni generates a caption that states the tow rope is attached to a boat.
With an atomic enumeration prompt and the same decoding setting, the model describes the tow rope as connected to an unseen source and does not introduce the absent boat.

Figure~\ref{fig:spurious}~(bottom) evaluates this effect distributionally.
We compute the co-occurrence hallucination rate (CHR), defined as the fraction of generations that mention the absent object $b$ when the cue $a$ is present, and show four representative pairs from the nine-pair probe set described in Appendix~\ref{app:exp:datasets}.
We sort the pairs by their COCO co-occurrence frequency $P(b \mid a)$.
Free-form generation has a higher CHR than atomic enumeration for every pair, and atomic enumeration reduces CHR by about $2.6\times$.
Prior work reports similar hallucination driven by co-occurrence on other backbones~\citep{datta2025evaluating,li2023evaluating}, which suggests that this behavior is not specific to one model.
A likely mechanism is that free-form generation conditions each token on the generated text, so early claims can activate training priors and lead to unsupported content~\citep{rohrbach2018,zhou2024analyzing,liu2026dual}.
Atomic enumeration constrains the model to produce short factual units, which reduces this generation path.
\begin{figure}[t]
  \centering
  \includegraphics[width=\columnwidth]{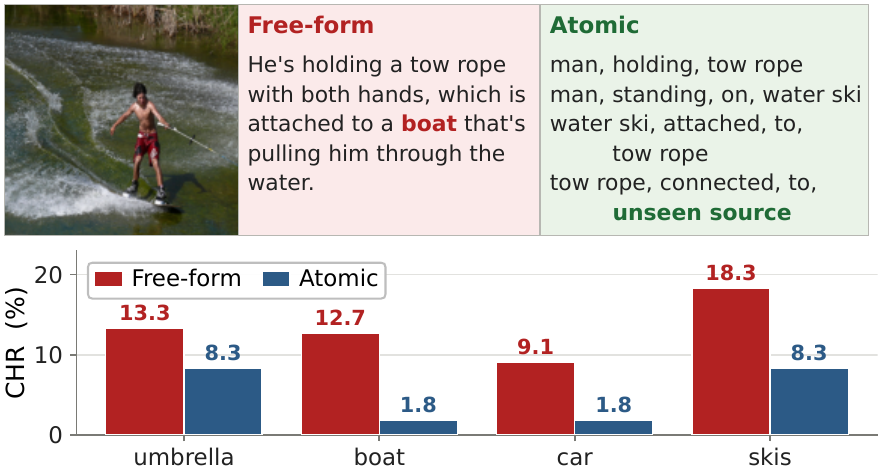}
  \caption{%
    \textbf{Free-form generation vs.\ atomic enumeration.}
    Top: a qualitative example where free-form generation adds an unsupported object, while atomic enumeration avoids it. Bottom: co-occurrence hallucination rate (CHR) across cue-to-absent object pairs.%
  }
  \label{fig:spurious}
  \vspace{-10pt}
\end{figure}

This observation motivates TIGER.
If spurious correlation enters through joint generation, a feedback mechanism that also relies on joint generation can inherit the same bias.
The next section uses this observation to redesign the feedback loop with independent fact extraction and risk-based repair.

\section{Methodology}
\label{sec:method}

\subsection{Problem Setup}
\label{sec:setup}

We consider a multimodal-to-text generation task where the input
$\mathbf{X}\in\mathcal{M}_X$ may be image, text, audio, or video, and
the output $\mathbf{Y}\in\mathcal{T}$ is textual. Given a task prompt
$\mathcal{P}_{\text{gen}}$, a multimodal model $\Phi$ realizes
\begin{equation}
\label{eq:single-shot}
\mathbf{Y}=\Phi(\mathcal{P}_{\text{gen}},\mathbf{X}).
\end{equation}
As shown in Section~\ref{sec:motivation}, single-shot generation can
inject spurious correlations into $\mathbf{Y}$ through joint
autoregressive decoding. Prior work often uses an iterative repair form:
\begin{equation}
\label{eq:iteration}
\mathbf{Y}_0=\Phi(\mathcal{P}_{\text{gen}},\mathbf{X}),\qquad
\mathbf{Y}_{t+1}=\Phi(\mathcal{P}_{\text{refine}},\mathbf{X},
\mathbf{Y}_t,\mathcal{F}_t),
\end{equation}
where $\mathcal{F}_t$ indicates what in $\mathbf{Y}_t$ needs correction.
A common critique-then-revise instantiation, such as Volcano, generates
natural-language feedback by jointly conditioning on the input and the
current output:
\begin{equation}
\label{eq:naive-feedback}
\mathcal{F}_t=\Phi(\mathcal{P}_{\text{fb}},\mathbf{X},\mathbf{Y}_t).
\end{equation}
Other correction methods use different response-conditioned signals,
such as tool-based visual validation in Woodpecker and generative visual
feedback in DeGF. However, existing methods do not independently
construct an observation graph from $\mathbf{X}$ and a claim graph from
$\mathbf{Y}_t$, nor do they compute deterministic per-claim risk scores.

This creates two limitations. First, critique feedback in
Eq.~\eqref{eq:naive-feedback} can inherit spurious correlations because
the feedback generator observes $\mathbf{Y}_t$, so unsupported claims may
be endorsed rather than flagged. 
Second, existing feedback or correction signals are not explicit
fact-level risk scores, so there is no direct way to rank claims or
to select them for repair under a compute budget.

\subsection{Redesigning \texorpdfstring{$\mathcal{F}_t$}{Ft}}
\label{sec:redesign}

We address these limitations by replacing Eq.~\eqref{eq:naive-feedback}
with two changes. First, $\mathbf{X}$ and $\mathbf{Y}_t$ are independently
projected into fact graphs:
\begin{equation}
\label{eq:extract}
G_{\mathbf{X}}=\Phi(\mathcal{P}_{\text{ext}},\mathbf{X}),\qquad
G_{\mathbf{Y}_t}=\Phi(\mathcal{P}_{\text{ext}},\mathbf{Y}_t).
\end{equation}
Second, feedback is replaced by a graph-conditioned selection operator:
\begin{equation}
\label{eq:tiger-feedback}
\mathcal{F}_t=\Psi_\alpha(G_{\mathbf{X}},G_{\mathbf{Y}_t}).
\end{equation}
Input isolation in Eq.~\eqref{eq:extract} mitigates the conditioning
bias, while $\alpha$-budgeted selection in $\Psi_\alpha$ enables
fact-level scheduling. The refine step in Eq.~\eqref{eq:iteration}
remains stochastic.

\noindent\textbf{Atomic extraction.}
The prompt $\mathcal{P}_{\text{ext}}$ asks $\Phi$ to enumerate
observable facts as triples $(e_1, \text{rel}, e_2)$, where $e_1$,
$e_2$ are head and tail entities and $\text{rel}$ is drawn from a
fixed relation vocabulary (Appendix~\ref{app:prompts}).
Each fact is generated within the boundary of a single triple, encouraging $\Phi$
to operate in local context and reducing the co-occurrence channel
documented in Section~\ref{sec:motivation}.
The input graph $G_{\mathbf{X}}$ and the output graph
$G_{\mathbf{Y}_t}$ are extracted in complete isolation.
$G_{\mathbf{X}}$ is computed once before any $\mathbf{Y}$ exists and remains fixed,
so errors in $\mathbf{Y}_t$ cannot directly affect the extraction of input facts,
blocking the direct conditioning path. Facts that share the same entity are connected by
coreference edges, so each graph carries both node-level and
edge-level structure.

\noindent\textbf{Risk-based selection $\Psi_\alpha$.}
For each claim $f \in G_{\mathbf{Y}_t}$, $\Psi_\alpha$ computes a
support score and a conflict score. The similarity between two facts
is computed as the mean of per-field clipped cosine similarities using a
frozen sentence transformer~\citep{reimers2019sentencebert}. The
\emph{local support}
$s_0(f) = \max_{g \in G_{\mathbf{X}}} \mathrm{sim}(f, g)$ measures
the best match in the observation graph. To compensate for extraction
omissions, local support is propagated along coreference edges in
$G_{\mathbf{Y}_t}$ with geometric decay:
\begin{equation}
\label{eq:support}
s(f) = \max_{f' \in \{f\} \cup \mathcal{N}_K(f)}\;
       \gamma^{\,d(f,\,f')} \cdot s_0(f'),
\end{equation}
where $\mathcal{N}_K(f)$ is the $K$-hop coreference neighborhood,
$d(f, f')$ is the hop distance, and $\gamma \in (0, 1)$ is a decay
factor. This propagation allows facts about the same entity to
reinforce each other, reducing false positives caused by extraction
omissions.
The \emph{conflict score} measures the strongest contradiction
between a claim and the observation graph. The function
$\mathrm{conflict}(f, g)$ is defined as the product of topic
consistency (how well subjects and predicates match) and conclusion
divergence (how much the objects differ), which acts as a soft gate:
conflict is significant only when two facts discuss the same topic
but reach different conclusions (Appendix~\ref{app:risk}). The
conflict score of a claim is:
\begin{equation}
\label{eq:conflict}
c(f) = \max_{g \in G_{\mathbf{X}}}\; \mathrm{conflict}(f, g).
\end{equation}

The risk combines lack of support and presence of conflict:
\begin{equation}
\label{eq:risk}
r(f) = (1 - s(f)) + \lambda \cdot c(f), \qquad \lambda > 0.
\end{equation}
Since the scoring components are deterministic, a fixed pair
$(G_{\mathbf{X}}, G_{\mathbf{Y}_t})$ yields the same per-claim
risk scores $r(f)$ for all $f \in G_{\mathbf{Y}_t}$. With
$N = |G_{\mathbf{Y}_t}|$, $\Psi_\alpha$ returns the
$\lceil \alpha N \rceil$ highest-risk facts:
\begin{equation}
\label{eq:H}
\Psi_\alpha(G_{\mathbf{X}}, G_{\mathbf{Y}_t})
\;=\;
\operatorname*{arg\,max}_{\substack{\mathcal{S} \,\subseteq\, G_{\mathbf{Y}_t} \\
                                    |\mathcal{S}| \,=\, \lceil \alpha N \rceil}}
\sum_{f \in \mathcal{S}} r(f).
\end{equation}
Each selected fact keeps its source position, so $\Phi$ repairs the
corresponding span locally. The full design and property verification
of $r(f)$ are given in Appendix~\ref{app:risk}.

\subsection{Algorithm and Convergence}
\label{sec:algorithm}

The complete iterative repair procedure is summarized in
Algorithm~\ref{alg:repair}.

\begin{algorithm}[t]
\caption{Iterative Fact-Level Repair}
\label{alg:repair}
\begin{algorithmic}[1]
\Require input $\mathbf{X}$; backbone $\Phi$; rounds $T$;
         budget $\alpha \in (0,1]$; risk weight $\lambda$
\Ensure repaired output $\mathbf{Y}_T$
\State $G_{\mathbf{X}} \gets \Phi(\mathcal{P}_{\text{ext}}, \mathbf{X})$
\State $\mathbf{Y}_0 \gets \Phi(\mathcal{P}_{\text{gen}}, \mathbf{X})$
\For{$t = 0, 1, \dots, T-1$}
    \State $G_{\mathbf{Y}_t} \gets \Phi(\mathcal{P}_{\text{ext}}, \mathbf{Y}_t)$
    \For{each fact $f \in G_{\mathbf{Y}_t}$}
        \State $s_0(f) \gets \max_{g \in G_{\mathbf{X}}}\ \mathrm{sim}(f, g)$
        \State $s(f) \gets \max_{f' \in \{f\} \cup \mathcal{N}_K(f)}\
            \gamma^{d(f,f')} \cdot s_0(f')$
        \State $c(f) \gets \max_{g \in G_{\mathbf{X}}}\ \mathrm{conflict}(f, g)$
        \State $r(f) \gets (1 - s(f)) + \lambda \cdot c(f)$
    \EndFor
    \State $\mathcal{F}_t \gets \text{top-}\lceil\alpha N\rceil \text{ facts of } G_{\mathbf{Y}_t} \text{ by } r(\cdot)$
    \State $\mathbf{Y}_{t+1} \gets \Phi(\mathcal{P}_{\text{refine}},\ \mathbf{X},\ \mathbf{Y}_t,\ \mathcal{F}_t)$
\EndFor
\State \Return $\mathbf{Y}_T$
\end{algorithmic}
\end{algorithm}

Let $R^{(t)} = \sum_{f \in G_{\mathbf{Y}_t}} r(f)$ denote the measured
total risk. Under explicit assumptions (bounded graph size, per-fact
repair progress $\varepsilon$, bounded side effects $\beta$, bounded
extraction loss $\xi$; see Appendix~\ref{app:proof}), the following
conditional bound holds.

\begin{theorem}[Conditional geometric risk bound]
\label{thm:convergence}
After $T$ rounds of repair, the expected measured risk satisfies
\begin{equation}
\label{eq:convergence}
\mathbb{E}\!\left[R^{(T)}\right]
\;\le\;
(1 - \alpha\varepsilon)^{T}\, R^{(0)}
\;+\;
\frac{\beta + \xi}{\alpha\varepsilon}.
\end{equation}
\end{theorem}

\noindent The full proof is in Appendix~\ref{app:proof}. The asymptotic floor $(\beta+\xi)/(\alpha\varepsilon)$ reflects the
residual risk caused by side effects and extraction loss. Notably, because each refine
step accesses raw input $\mathbf{X}$ rather than only
$G_{\mathbf{X}}$, the repair loop can recover facts lost during
extraction, so $\mathbf{Y}_T$ can exceed the coverage of
$G_{\mathbf{X}}$. We verify this empirically in
Section~\ref{subsec:mechanism}.



\section{Experimental Evaluation}
\label{sec:experiments}

We structure our evaluation around six research questions.
\textbf{RQ1. Does \method{} reduce hallucination across cross-modal
generation paths?} We evaluate \method{} against ten baselines on
COCO, AMBER, MMHal-Bench, Clotho, and VideoHallucer, which cover
image$\to$text, image+text$\to$text, audio$\to$text, and
video$\to$text, and we report external hallucination and quality
metrics over three seeds.
\textbf{RQ2. Is each core component of \method{} necessary?} We
ablate iterative repair, atomic projection, and deterministic risk
ranking under the same refinement step, and we compare four feedback
levels from direct decoding to full \method{}.
\textbf{RQ3. How sensitive is \method{} to its main
hyperparameters?} We run one-at-a-time sweeps over the number of
repair rounds $T$, the batch budget $\alpha$, and the conflict
weight $\lambda$, and we report CHAIR$_s$ around the default
setting.
\textbf{RQ4. Why does \method{} work?} We probe
whether atomic projection reduces spurious correlation during
feedback generation, and whether the refinement step can recover
correct facts that are missing from the extracted input graph
$G_{\mathbf{X}}$.
\textbf{RQ5. Does \method{} transfer to real-world inputs?} We apply
\method{} to multi-source crisis reporting on CrisisFACTS, where the
input combines noisy text streams and images from five disaster
events, and we report gold-match precision, recall, and F1.
\textbf{RQ6. What does \method{} cost, and can the cost be reduced?}
We measure backbone calls and wall-clock time per sample against all
baselines, and we evaluate two lower-cost operating points: fewer
repair rounds and an early-stopping variant.

\subsection{Experimental Setup}
\label{subsec:setup}

\noindent\textbf{Datasets.}
We evaluate \method{} on five benchmarks that cover four cross-modal generation paths and four modalities:
(1)~\textbf{COCO Captions}~\cite{chen2015microsoft} for image$\to$text generation;
(2)~\textbf{MMHal-Bench}~\cite{mmhalbench2023} for image+text$\to$text generation;
(3)~\textbf{AMBER}~\cite{wang2023amber} for image$\to$text hallucination evaluation;
(4)~\textbf{Clotho}~\cite{drossos2020clotho} for audio$\to$text generation;
and (5)~\textbf{VideoHallucer}~\cite{wang2024videohallucer} for video$\to$text hallucination evaluation.
We report the results as mean$\pm$std over three random seeds. Additional dataset details are provided in Appendix~\ref{app:exp:datasets}.

\noindent\textbf{Baselines.}
We compare \method{} with Frozen and three groups of baselines using the same frozen backbone whenever applicable. Frozen uses direct decoding from the frozen backbone without post-processing. The resampling baselines include BoN+CLIP, BoN+VisualPRM~\cite{visualprm2025}, and BoN+CycleReward~\cite{cyclereward2025}, which select one output from multiple candidates. The modality-specific decoding baselines include VCD~\cite{leng2024mitigating} for image inputs, AAD~\cite{hsu2025reducing} for audio inputs, and TCD~\cite{zhang2024eventhallusion} for video inputs. The iterative refinement baselines include Woodpecker~\cite{yin2024woodpecker}, DeGF~\cite{degf2025}, and Volcano~\cite{lee2024volcano}. The iterative refinement methods follow the joint-conditioning feedback template in Eq.~\eqref{eq:naive-feedback}, while the resampling and decoding methods do not produce fact-level feedback. Prompt templates and implementation details are provided in Appendices~\ref{app:prompts:baselines} and~\ref{app:exp:models}.
\noindent\textbf{Metrics.}
To avoid self-referential evaluation, every headline metric is computed independently of \method{}'s evidence graph. Table~\ref{tab:metrics} lists the external metrics used for each evaluation path.

\begin{table}[t]
\centering
\small
\caption{External metrics for each evaluation path.}
\label{tab:metrics}
\begin{tabular}{@{}ll@{}}
\toprule
Path & Metrics \\
\midrule
COCO          & CHAIR$_s$$\downarrow$, CHAIR$_i$$\downarrow$, BERTScore$\uparrow$ \\
AMBER         & Disc.\ Acc.$\uparrow$, CHAIR$_g$$\downarrow$ \\
MMHal-Bench   & BERTScore$\uparrow$ \\
Clotho        & RougeL$\uparrow$, BLEU$\uparrow$, CLAP$\uparrow$, AEHR$\downarrow$ \\
VideoHallucer & HallucRate$\downarrow$, Acc$_\text{b}$$\uparrow$, Acc$_\text{h}$$\uparrow$, Paired$\uparrow$ \\
\bottomrule
\end{tabular}
\end{table}

\noindent\textbf{Implementation details.}
The primary backbone $\Phi$ is Qwen2.5-Omni-7B~\cite{qwen25omni2025}. The frozen model serves as both the generator of $\mathbf{Y}_t$ and the extractor that produces $G_{\mathbf{X}}$ and $G_{\mathbf{Y}_t}$ through Eq.~\eqref{eq:extract}. The extraction prompt $\mathcal{P}_{\text{ext}}$ asks $\Phi$ to enumerate facts as $(e_1, \text{rel}, e_2)$ triples. It is applied once to $\mathbf{X}$ before the repair loop and then to each $\mathbf{Y}_t$ during repair. Per-dataset hyperparameter settings, including repair rounds $T$, batch budget $\alpha$, and conflict weight $\lambda$, are listed in Appendix~\ref{app:exp:protocol}. We report results with LLaVA-1.5-7B~\cite{liu2024improved} to evaluate cross-backbone generalization. Additional results with GPT-5.5, Gemini 3.5 Flash, and Claude Haiku~4.5 are reported in Appendix~\ref{app:api_backbones}. Appendix~\ref{app:extractor_variants} compares \method{} against deterministic modality-specific extractors.

\subsection{Main Results (RQ1)}
\label{subsec:main}

To answer RQ1, we evaluate whether \method{} reduces hallucination across image$\to$text, image+text$\to$text, audio$\to$text, and video$\to$text generation paths. Table~\ref{tab:main_image} reports results on COCO, AMBER, and MMHal-Bench with Qwen2.5-Omni-7B and LLaVA-1.5-7B. Table~\ref{tab:main_audio} reports audio$\to$text results on Clotho, and Table~\ref{tab:main_video} reports video$\to$text results on VideoHallucer.

\begin{figure}[t]
\centering
\includegraphics[width=0.9\columnwidth]{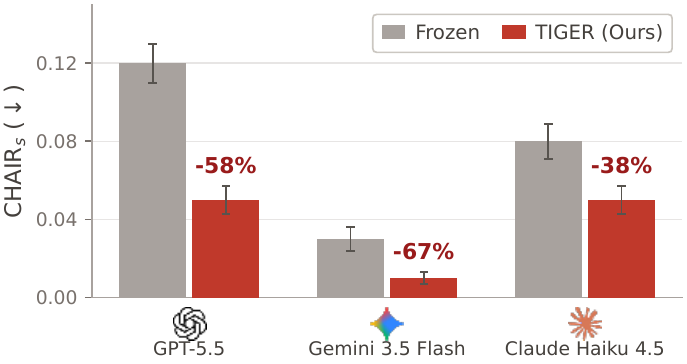}
\caption{\method{} on three proprietary API backbones (COCO
image$\to$text). \method{} reduces CHAIR$_s$ by 38--67\% relative to
Frozen. Full
results are in Appendix~\ref{app:api_backbones}.}
\label{fig:api_backbones}
\end{figure}

\begin{table*}[t]
\centering
\scriptsize
\setlength{\tabcolsep}{0.9pt}
\renewcommand{\arraystretch}{1.1}
\caption{Main results across two paths and two backbones. The image$\to$text path uses COCO and AMBER ; the image+text$\to$text path uses MMHal-Bench. \textbf{Bold}: best per column within each backbone block. \colorbox{tigerblue}{Blue}: \method{}.}
\label{tab:main_image}
\begin{tabular}{@{}l|ccc|cc|c|ccc|cc|c@{}}
\toprule
 & \multicolumn{6}{c|}{\textbf{Qwen2.5-Omni-7B}} & \multicolumn{6}{c}{\textbf{LLaVA-1.5-7B}} \\
\cmidrule(lr){2-7} \cmidrule(lr){8-13}
 & \multicolumn{3}{c|}{COCO} & \multicolumn{2}{c|}{AMBER} & MMHal & \multicolumn{3}{c|}{COCO} & \multicolumn{2}{c|}{AMBER} & MMHal \\
\cmidrule(lr){2-4}\cmidrule(lr){5-6}\cmidrule(lr){7-7}\cmidrule(lr){8-10}\cmidrule(lr){11-12}\cmidrule(lr){13-13}
\textbf{Method} & CHAIR$_s$$\downarrow$ & CHAIR$_i$$\downarrow$ & BERT$\uparrow$ & Disc.\ Acc$\uparrow$ & CHAIR$_g$$\downarrow$ & BERT$\uparrow$ & CHAIR$_s$$\downarrow$ & CHAIR$_i$$\downarrow$ & BERT$\uparrow$ & Disc.\ Acc$\uparrow$ & CHAIR$_g$$\downarrow$ & BERT$\uparrow$ \\
\midrule
Frozen      & .070$_{\pm.008}$ & .069$_{\pm.006}$ & .588$_{\pm.003}$ & 77.0$_{\pm1.3}$ & 5.0$_{\pm.3}$ & .703$_{\pm.009}$ & .060$_{\pm.006}$ & .028$_{\pm.003}$ & .740$_{\pm.003}$ & 70.3$_{\pm1.3}$ & 8.6$_{\pm.4}$ & .757$_{\pm.017}$ \\
BoN+CLIP    & .080$_{\pm.005}$ & .047$_{\pm.004}$ & .613$_{\pm.001}$ & 78.0$_{\pm1.2}$ & 7.4$_{\pm.4}$ & .717$_{\pm.007}$ & .060$_{\pm.006}$ & .025$_{\pm.003}$ & .746$_{\pm.003}$ & 73.0$_{\pm1.4}$ & 8.3$_{\pm.3}$ & .779$_{\pm.017}$ \\
BoN+VisPRM  & .060$_{\pm.008}$ & .031$_{\pm.004}$ & .616$_{\pm.002}$ & 79.7$_{\pm1.2}$ & 5.5$_{\pm.3}$ & .714$_{\pm.021}$ & .050$_{\pm.006}$ & .021$_{\pm.003}$ & .746$_{\pm.003}$ & 74.3$_{\pm1.4}$ & 8.7$_{\pm.4}$ & .781$_{\pm.017}$ \\
BoN+CycRew  & .130$_{\pm.011}$ & .082$_{\pm.007}$ & .598$_{\pm.003}$ & 78.0$_{\pm1.2}$ & 5.3$_{\pm.3}$ & .719$_{\pm.011}$ & .040$_{\pm.003}$ & .020$_{\pm.003}$ & .745$_{\pm.003}$ & 72.3$_{\pm1.4}$ & 12.2$_{\pm.5}$ & .777$_{\pm.016}$ \\
Woodpecker  & .100$_{\pm.009}$ & .067$_{\pm.014}$ & .616$_{\pm.003}$ & 79.0$_{\pm.8}$ & 5.1$_{\pm.3}$ & .706$_{\pm.017}$ & .050$_{\pm.007}$ & .025$_{\pm.003}$ & .740$_{\pm.004}$ & 76.0$_{\pm1.3}$ & 7.9$_{\pm.3}$ & .672$_{\pm.025}$ \\
DeGF        & .060$_{\pm.008}$ & .033$_{\pm.003}$ & .594$_{\pm.003}$ & 68.3$_{\pm.8}$ & 5.0$_{\pm.3}$ & .706$_{\pm.069}$ & .040$_{\pm.006}$ & .018$_{\pm.003}$ & .771$_{\pm.003}$ & 65.7$_{\pm1.5}$ & 8.4$_{\pm.4}$ & .780$_{\pm.018}$ \\
VCD         & .060$_{\pm.007}$ & .045$_{\pm.004}$ & .610$_{\pm.003}$ & 78.5$_{\pm1.2}$ & 4.8$_{\pm.3}$ & .720$_{\pm.010}$ & .045$_{\pm.006}$ & .020$_{\pm.003}$ & .755$_{\pm.003}$ & 72.0$_{\pm1.4}$ & 8.0$_{\pm.4}$ & .770$_{\pm.018}$ \\
Volcano     & .070$_{\pm.008}$ & .042$_{\pm.004}$ & .637$_{\pm.004}$ & 78.0$_{\pm.8}$ & 5.2$_{\pm.3}$ & .742$_{\pm.008}$ & .040$_{\pm.006}$ & .023$_{\pm.004}$ & .753$_{\pm.004}$ & 69.3$_{\pm1.5}$ & 8.2$_{\pm.3}$ & .767$_{\pm.019}$ \\
\rowcolor{tigerblue}
\textbf{\method{}} & $\mathbf{.050}_{\pm.007}$ & $\mathbf{.035}_{\pm.005}$ & $\mathbf{.643}_{\pm.003}$ & $\mathbf{82.0}_{\pm.8}$ & $\mathbf{4.3}_{\pm.2}$ & $\mathbf{.756}_{\pm.020}$ & $\mathbf{.030}_{\pm.006}$ & $\mathbf{.013}_{\pm.003}$ & $\mathbf{.772}_{\pm.003}$ & $\mathbf{79.7}_{\pm1.5}$ & $\mathbf{7.4}_{\pm.4}$ & $\mathbf{.790}_{\pm.018}$ \\
\bottomrule
\end{tabular}%
\end{table*}

As shown in Table~\ref{tab:main_image}, \method{} reduces
hallucination on image-based tasks while preserving semantic quality.
On Qwen2.5-Omni-7B, it reduces COCO CHAIR$_s$ from $0.070$ to $0.050$
and CHAIR$_i$ from $0.069$ to $0.035$, while improving BERTScore from
$0.588$ to $0.643$. On LLaVA-1.5-7B, it lowers COCO CHAIR$_s$ from
$0.060$ to $0.030$ and achieves the highest BERTScore. The same trend
appears on AMBER and MMHal-Bench, where \method{} improves
faithfulness without degrading text quality. On MMHal-Bench, a
category-level diagnostic further shows that the gains concentrate on
relation, comparison, and environment questions
(Appendix~\ref{app:mmhal-cat}). These results suggest that the repair is not tied to a single backbone or benchmark.

Tables~\ref{tab:main_audio} and~\ref{tab:main_video} show that the
gains extend beyond image inputs. On Clotho, \method{} improves
captioning quality and audio-text alignment, and it reduces AEHR from
$0.803$ to $0.757$, which indicates fewer unsupported audio event
mentions. On VideoHallucer, \method{} lowers HallucRate from $0.015$
to $0.010$ and improves Paired accuracy from $0.170$ to $0.198$.
Figure~\ref{fig:api_backbones} further extends the comparison to
three proprietary API backbones, where \method{} reduces CHAIR$_s$
by 38--67\% relative to Frozen and also improves BERTScore in every
case (Appendix~\ref{app:api_backbones}). Overall, \method{}
consistently reduces unsupported claims across multimodal inputs and across five open-source and proprietary
backbones, while maintaining task performance.

\begin{table}[t]
\centering
\scriptsize
\setlength{\tabcolsep}{2pt}
\renewcommand{\arraystretch}{1.0}
\caption{Audio$\to$text on Clotho with Qwen2.5-Omni-7B. CLAP is reference-free caption-audio similarity; AEHR is the audio analogue of CHAIR$_i$. \textbf{Bold}: best per column. \colorbox{tigerblue}{Blue}: \method{}.}
\label{tab:main_audio}
\resizebox{\columnwidth}{!}{%
\begin{tabular}{@{}l@{\hskip 4pt}cccc@{}}
\toprule
Method & RougeL$\uparrow$ & BLEU$\uparrow$ & CLAP$\uparrow$ & AEHR$\downarrow$ \\
\midrule
Frozen      & .142$_{\pm.002}$ & .013$_{\pm.001}$ & .332$_{\pm.004}$ & .803$_{\pm.009}$ \\
BoN+CLAP    & .144$_{\pm.001}$ & .016$_{\pm.001}$ & .329$_{\pm.003}$ & .797$_{\pm.007}$ \\
BoN+CycRew  & .139$_{\pm.002}$ & .017$_{\pm.000}$ & .345$_{\pm.003}$ & .795$_{\pm.006}$ \\
AAD         & .140$_{\pm.002}$ & .005$_{\pm.000}$ & .364$_{\pm.003}$ & .777$_{\pm.008}$ \\
Volcano     & .143$_{\pm.001}$ & .010$_{\pm.001}$ & .343$_{\pm.003}$ & .800$_{\pm.005}$ \\
\rowcolor{tigerblue}
\textbf{\method{}} & $\mathbf{.147}_{\pm.001}$ & $\mathbf{.019}_{\pm.001}$ & $\mathbf{.376}_{\pm.004}$ & $\mathbf{.757}_{\pm.008}$ \\
\bottomrule
\end{tabular}%
}
\end{table}

\begin{table}[t]
\centering
\scriptsize
\setlength{\tabcolsep}{2pt}
\renewcommand{\arraystretch}{1.0}
\caption{Video$\to$text on VideoHallucer with Qwen2.5-Omni-7B. HallucRate is the analogue of CHAIR$_g$; Paired requires both halves of a pair correct. \textbf{Bold}: best per column. \colorbox{tigerblue}{Blue}: \method{}.}
\label{tab:main_video}
\resizebox{\columnwidth}{!}{%
\begin{tabular}{@{}l@{\hskip 4pt}cccc@{}}
\toprule
Method & HallucRate$\downarrow$ & Acc$_{\text{b}}\uparrow$ & Acc$_{\text{h}}\uparrow$ & Paired$\uparrow$ \\
\midrule
Frozen          & .015$_{\pm.003}$ & .210$_{\pm.014}$ & .495$_{\pm.016}$ & .170$_{\pm.012}$ \\
BoN+VideoCLIP   & .015$_{\pm.004}$ & .219$_{\pm.013}$ & .501$_{\pm.015}$ & .175$_{\pm.011}$ \\
BoN+CycRew      & .017$_{\pm.004}$ & .227$_{\pm.012}$ & .508$_{\pm.017}$ & .189$_{\pm.013}$ \\
TCD             & .016$_{\pm.004}$ & .220$_{\pm.013}$ & .505$_{\pm.016}$ & .180$_{\pm.012}$ \\
Volcano         & .015$_{\pm.005}$ & .250$_{\pm.015}$ & .490$_{\pm.017}$ & .175$_{\pm.010}$ \\
\rowcolor{tigerblue}
\textbf{\method{}} & $\mathbf{.010}_{\pm.003}$ & $\mathbf{.260}_{\pm.013}$ & $\mathbf{.526}_{\pm.015}$ & $\mathbf{.198}_{\pm.011}$ \\
\bottomrule
\end{tabular}%
}
\end{table}

\subsection{Ablation: Are All Three Components Necessary? (RQ2)}
\label{subsec:ablation}

To answer RQ2, we conduct a component ablation that tests iterative repair, atomic projection, and deterministic risk ranking under the same refinement step $\mathbf{Y}_{t+1}=\Phi(\mathcal{P}_{\text{refine}},\mathbf{X},\mathbf{Y}_t,\mathcal{F}_t)$. We define four levels that differ only in how the feedback signal $\mathcal{F}_t$ is produced. \textbf{L0 (Frozen)} uses direct decoding from the backbone without repair, where $\mathbf{Y}=\Phi(\mathcal{P}_{\text{gen}},\mathbf{X})$. \textbf{L1 (Naive feedback)} adds iterative repair with joint-conditioning feedback $\mathcal{F}_t=\Phi(\mathcal{P}_{\text{fb}},\mathbf{X},\mathbf{Y}_t)$ as in Eq.~\eqref{eq:naive-feedback}. \textbf{L2 (Text feedback)} adds atomic projection by extracting $G_{\mathbf{X}}$ and $G_{\mathbf{Y}_t}$ independently through Eq.~\eqref{eq:extract}, but it still uses text feedback without deterministic risk ranking. \textbf{L3 (\method{})} adds deterministic risk ranking and sets $\mathcal{F}_t=\Psi_\alpha(G_{\mathbf{X}},G_{\mathbf{Y}_t})$ as in Eq.~\eqref{eq:tiger-feedback}. Thus, L0$\to$L1 tests iterative repair alone, L1$\to$L2 isolates atomic projection, and L2$\to$L3 isolates deterministic risk ranking.

Figure~\ref{fig:ablation} reports the ablation results on COCO. Iterative repair alone does not improve the output: L1 increases CHAIR$_s$ from $0.070$ to $0.080$ and decreases BERTScore from $0.588$ to $0.528$. This result supports our analysis that a repair loop can amplify hallucinations when its feedback is produced by joint conditioning on the input and output. Adding atomic projection in L2 reduces this bias and restores performance, raising BERTScore to $0.603$ while bringing CHAIR$_s$ and CHAIR$_i$ back to the Frozen level. Adding deterministic risk ranking in L3 gives the best result on all three metrics, with CHAIR$_s=0.050$, CHAIR$_i=0.035$, and BERTScore $=0.643$. Results show that the three components play different roles: iterative repair provides the correction mechanism, atomic projection improves the reliability of feedback, and deterministic risk ranking directs the repair budget to the facts most likely to be unsupported.

\begin{figure}[t]
\centering
\includegraphics[width=\columnwidth]{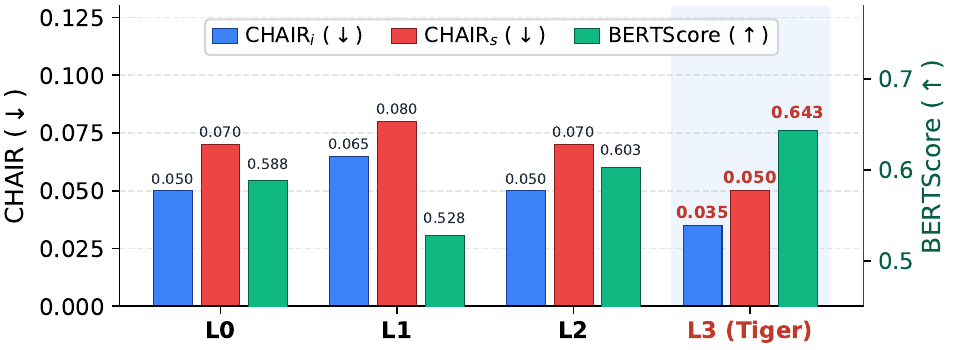}
\caption{Component ablation on COCO. }
\label{fig:ablation}
\end{figure}

\subsection{Hyperparameter Sensitivity (RQ3)}
\label{subsec:sensitivity}

To answer RQ3, we conduct a one-at-a-time sensitivity analysis on COCO for the three main hyperparameters of \method{}: the number of repair rounds $T$, the repair budget $\alpha$, and the conflict weight $\lambda$. For each sweep, we vary one hyperparameter and keep the other two fixed at the default setting. Figure~\ref{fig:sensitivity} reports CHAIR$_s$, which measures object hallucination at the sentence level and is independent of the internal risk score used by \method{}.

The results show that \method{} is stable near the default setting. Increasing $T$ reduces CHAIR$_s$ substantially at first, but the curve flattens after a small number of repair rounds. This suggests that most correctable high-risk facts are handled early in the repair process. The repair budget $\alpha$ shows a trade-off between coverage and stability. A small budget may leave some risky facts unrepaired, while a large budget can revise too many facts in one round. The conflict weight $\lambda$ also requires balance. When $\lambda=0$, the selector ignores direct contradictions; when $\lambda$ is too large, conflict signals can dominate the ranking. Overall, the default setting lies in a low-CHAIR$_s$ region across all three sweeps, which indicates that the method is not sensitive to small changes in its main hyperparameters. Appendix~\ref{app:unified} further shows that a single unified
configuration transfers across all five benchmarks with results
close to the per-dataset operating points, and that the COCO
sensitivity trend also holds on Clotho.

\begin{figure}[t]
\centering
\includegraphics[width=\linewidth]{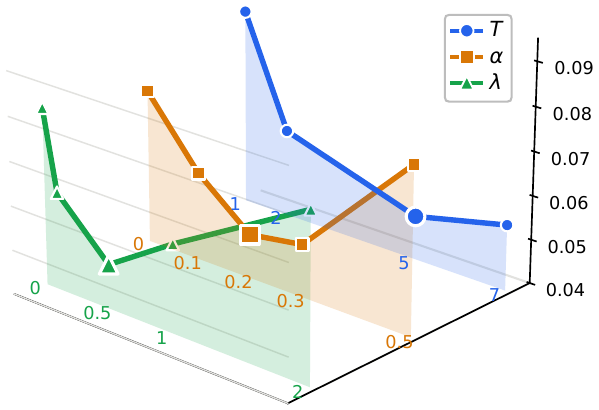}
\caption{Hyperparameter sensitivity of \method{} on COCO. Each curve varies one hyperparameter while keeping the other two fixed. The vertical axis reports CHAIR$_s$; lower is better.}
\label{fig:sensitivity}
\end{figure}

\subsection{Mechanism: Projection and Refinement (RQ4)}
\label{subsec:mechanism}

To answer RQ4, we test two mechanisms that distinguish \method{} from joint-feedback baselines: whether atomic projection reduces spurious correlation in feedback, and whether the refinement step can recover correct facts that are missing from the extracted input graph $G_{\mathbf{X}}$. These probes focus on the feedback stage because this is where unsupported claims can either be removed or reinforced.

\begin{figure*}[t]
\centering
\begin{subfigure}{0.32\linewidth}
  \centering
  \includegraphics[width=\linewidth]{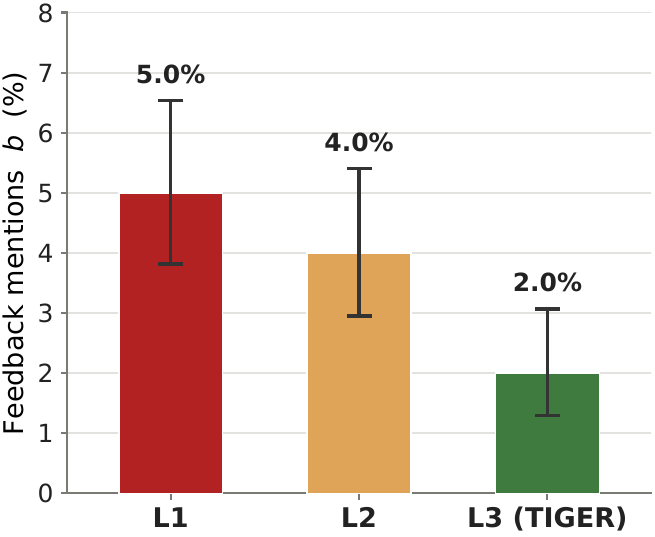}
  \caption{Feedback mention rate.}
  \label{fig:rq4_fer}
\end{subfigure}\hfill
\begin{subfigure}{0.32\linewidth}
  \centering
  \includegraphics[width=\linewidth]{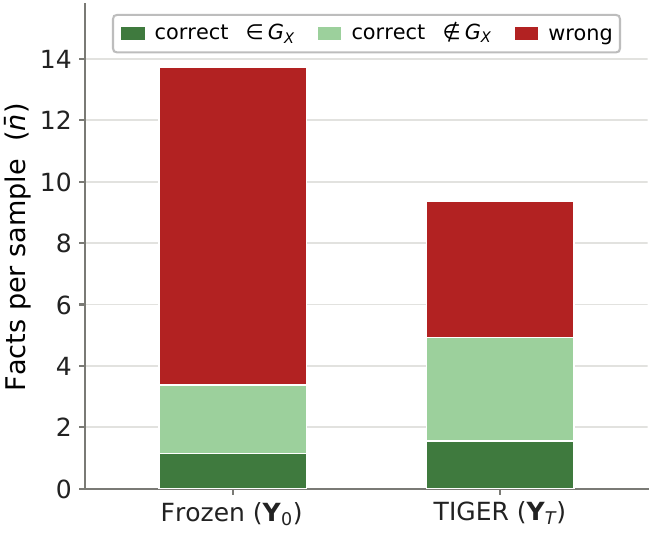}
  \caption{Fact composition of $G_{\mathbf{Y}}$.}
  \label{fig:rq4_composition}
\end{subfigure}\hfill
\begin{subfigure}{0.32\linewidth}
  \centering
  \includegraphics[width=\linewidth]{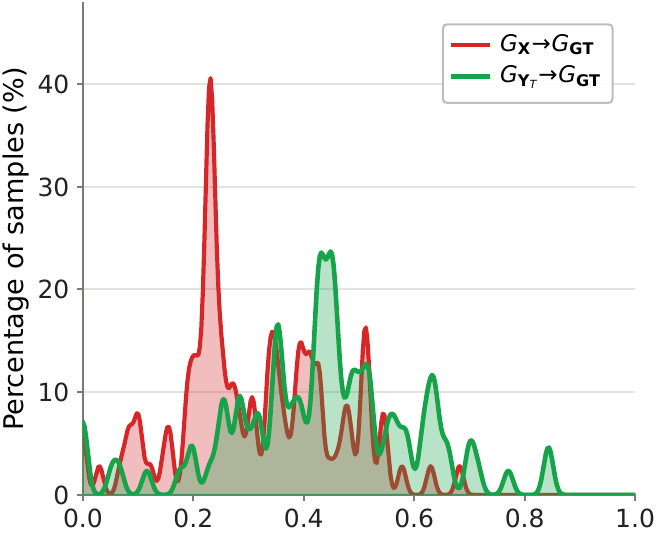}
  \caption{Sample-level similarity to $G_{\mathbf{GT}}$.}
  \label{fig:rq4_density}
\end{subfigure}
\caption{Mechanism analysis on COCO val2014. (a) Joint feedback channels mention the absent object $b$ more often than the atomic channel. (b) Compared with Frozen, \method{} preserves correct facts, recovers additional correct facts beyond $G_{\mathbf{X}}$, and removes wrong facts. (c) The repaired output shifts toward higher similarity to $G_{\mathbf{GT}}$, which shows that refinement recovers content that the extractor misses.}
\label{fig:rq4}
\end{figure*}

We first test whether atomic projection reduces the spurious correlation that joint feedback can inherit from the backbone. We select $1000$ images from COCO val2014 where each image contains a scene cue $a$ (e.g., beach or kitchen) and does not contain an object $b$ that often co-occurs with $a$ in COCO captions. We compare three feedback channels under the same $T\!=\!5$ repair loop. L1 uses naive joint feedback, $\mathcal{F}_t=\Phi(\mathcal{P}_{\text{fb}},\mathbf{X},\mathbf{Y}_t)$, as in Eq.~\eqref{eq:naive-feedback}. L2 uses independently extracted graphs as text in the feedback prompt, but it does not perform deterministic risk ranking. L3 uses the atomic feedback of \method{}, $\Psi_\alpha(G_{\mathbf{X}},G_{\mathbf{Y}_t})$. We report the feedback mention rate, defined as the fraction of samples whose feedback mentions the absent object $b$. Figure~\ref{fig:rq4}(a) shows that L1 mentions $b$ in $5.0\%$ of samples and L2 in $4.0\%$, while L3 mentions $b$ in only $2.0\%$. This result indicates that joint feedback can reintroduce co-occurrence priors, whereas atomic feedback reduces this channel.

We next test whether refinement can recover correct content that the extractor misses. We label each fact in the repaired output as correct if it is supported by a reference graph $G_{\mathbf{GT}}$ extracted from the five captions of the image, and as in-extraction if it is supported by $G_{\mathbf{X}}$. Figure~\ref{fig:rq4}(b) reports the resulting fact composition. Compared with Frozen, \method{} keeps more correct facts, recovers additional correct facts beyond $G_{\mathbf{X}}$, and removes a large portion of wrong facts. This shows that refinement does not simply copy the extracted graph. Instead, it can still use the raw input $\mathbf{X}$ to recover facts that the extractor did not capture. Figure~\ref{fig:rq4}(c) gives a sample-level view of the same effect. The similarity from $G_{\mathbf{X}}$ to $G_{\mathbf{GT}}$ is concentrated around lower values, while the similarity from $G_{\mathbf{Y}_T}$ to $G_{\mathbf{GT}}$ shifts toward higher values. This shift shows that the repaired output covers more reference-supported content than the extracted input graph alone. Overall, these probes explain why \method{} improves over the L1 and L2 ablations: atomic projection reduces biased feedback, and input-grounded refinement preserves the ability to recover missing correct facts. 

This recovery indicates that the repair loop tolerates gaps in
$G_{\mathbf{X}}$. Two additional analyses examine the extent of this
tolerance. When we annotate sampled outputs by whether each claim is graph-representable, coverage ranges from 74\% to 87\% across COCO, MMHal-Bench, and Clotho, and \method{} improves over Frozen in every coverage bucket, including samples with coverage below 80\% (Appendix~\ref{app:coverage}). When we instead corrupt $G_{\mathbf{X}}$ directly, deleting or injecting 30\% of its facts, performance degrades gradually and remains better than Frozen (Appendix~\ref{app:perturb}).



\begin{table}[t]
\centering
\scriptsize
\setlength{\tabcolsep}{0pt}
\renewcommand{\arraystretch}{1.15}
\caption{CrisisFACTS case study on five disaster event-days. P/R/F1 are gold-match precision, recall, and F1. \textbf{Bold}: better score. \method{} columns are shaded \colorbox{tigerblue}{blue}.}
\label{tab:case_study}
\begin{tabular*}{\columnwidth}{@{\extracolsep{\fill}}l c >{\columncolor{tigerblue}}c c >{\columncolor{tigerblue}}c c >{\columncolor{tigerblue}}c@{}}
\toprule
& \multicolumn{2}{c}{P$\uparrow$} 
& \multicolumn{2}{c}{R$\uparrow$} 
& \multicolumn{2}{c}{F1$\uparrow$} \\
\cmidrule(lr){2-3}\cmidrule(lr){4-5}\cmidrule(lr){6-7}
Disaster & Frozen & \method{} & Frozen & \method{} & Frozen & \method{} \\
\midrule
Hurricane & .679 & $\mathbf{1.00}$ & .631 & $\mathbf{.679}$ & .654 & $\mathbf{.809}$ \\
Wildfire  & .714 & $\mathbf{.962}$ & .583 & $\mathbf{.652}$ & .642 & $\mathbf{.777}$ \\
Flood     & .623 & $\mathbf{.943}$ & .661 & $\mathbf{.707}$ & .641 & $\mathbf{.808}$ \\
Explosion & .583 & $\mathbf{.909}$ & .552 & $\mathbf{.594}$ & .567 & $\mathbf{.718}$ \\
Tornado   & .654 & $\mathbf{.978}$ & .587 & $\mathbf{.618}$ & .619 & $\mathbf{.757}$ \\
\bottomrule
\end{tabular*}
\end{table}

\subsection{Case Study: Multi-Source Crisis Reporting (RQ5)}
\label{subsec:case_study}

To answer RQ5, we test whether \method{} transfers from clean benchmark inputs to a noisy real-world setting. We use \emph{CrisisFACTS}~\cite{crisisfacts2022}, a NIST benchmark that contains Twitter, Reddit, Facebook, and online-news streams with assessor-annotated gold facts. We select one peak event-day from each of five disaster types: Hurricane, Wildfire, Flood, Explosion, and Tornado. For each event, the backbone receives text streams and six high-priority images, and it generates a situation report. Figure~\ref{fig:casestudy_pipeline} shows the Hurricane example, where the multi-source input is consolidated into an evidence graph $G_{\mathbf{X}}$ with $387$ facts and $339$ edges. This setting is substantially noisier and larger than the benchmark inputs, so it tests whether fact-level repair remains useful beyond controlled datasets.
\begin{figure}[t]
    \centering
    \includegraphics[width=\columnwidth]{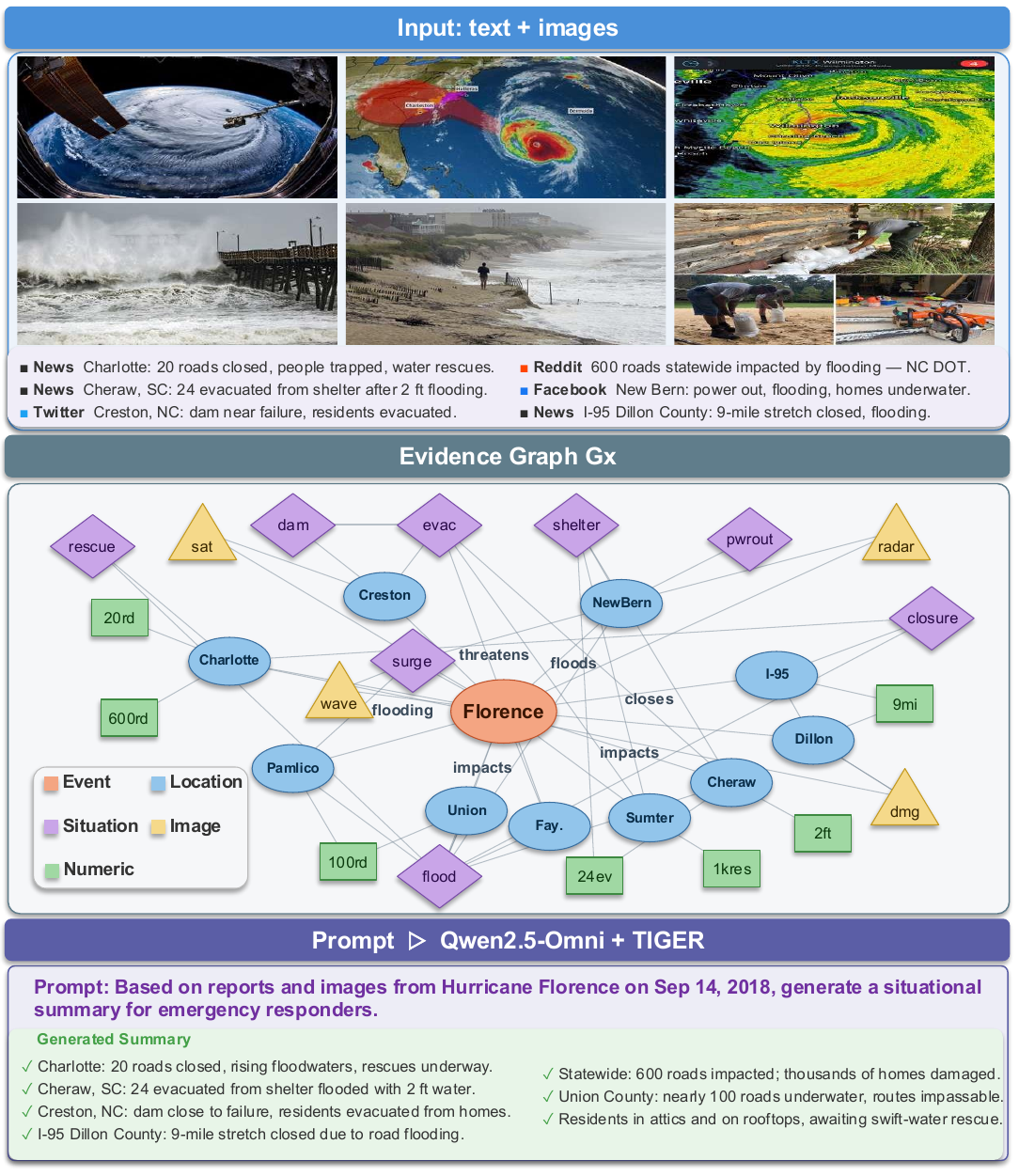}
\caption{Case-study pipeline on the Hurricane event. Multi-source inputs are consolidated into $G_{\mathbf{X}}$.}
    \label{fig:casestudy_pipeline}
    \vspace{-15pt}
\end{figure}

Table~\ref{tab:case_study} reports the Frozen and \method{} results on Qwen2.5-Omni-7B. This case study tests whether the repair mechanism can improve grounding in a realistic multi-source application. Across the five disasters, \method{} improves average Precision from $0.65$ to $0.96$, Recall from $0.60$ to $0.65$, and F1 from $0.62$ to $0.77$. The precision gain is especially important in this setting because situation reports should avoid claims that are not supported by the evidence. At the same time, the recall and F1 improvements show that \method{} does not simply remove uncertain content; it also preserves or recovers facts that are supported by the evidence graph. Qualitatively, the repaired reports replace vague statements with grounded entities and statistics, such as changing ``in some areas'' to ``Creston, North Carolina'' and ``numerous road closures'' to ``a 9-mile stretch of Interstate 95 in Dillon County.'' These results show that \method{} can support grounded summarization when the input contains noisy evidence from multiple sources.

\subsection{Computational Cost (RQ6)}
\label{subsec:cost}

To answer RQ6, we measure the inference cost of \method{} and test
whether lower-cost operating points retain its gains.
Figure~\ref{fig:compute_cost} reports wall-clock time per sample on
COCO under the same Qwen2.5-Omni-7B backbone on a single GPU.
\method{} at $T\!=\!5$ takes $199.3$\,s per sample, about
$5\!\times$ Frozen, but is faster than the other iterative methods,
Volcano ($230.0$\,s) and Woodpecker ($245.9$\,s). Since the
sensitivity curve in Figure~\ref{fig:sensitivity} flattens for
$T\!\geq\!3$, we evaluate two lower-cost operating points:
$T\!=\!3$, and an early-stopping variant that terminates when the
mean per-fact risk stops decreasing (Appendix~\ref{app:earlystop}).
As shown in Table~\ref{tab:cost}, $T\!=\!3$ retains most of the
reduction in hallucination while cutting backbone calls from 12 to
8, and early stopping reduces the average number of rounds to 2.71
with limited loss. These operating points make the cost--quality
trade-off explicit; further details are given in Appendix~\ref{app:earlystop}.

\begin{table}[t]
\centering
\caption{Cost--quality trade-off on COCO with Qwen2.5-Omni-7B.
Wall-clock is per sample on a single GPU. \textbf{Bold}: best per
column. \colorbox{tigerblue}{Blue}: early-stopping.}
\label{tab:cost}
\footnotesize
\setlength{\tabcolsep}{3pt}
\renewcommand{\arraystretch}{1.15}
\resizebox{\columnwidth}{!}{%
\begin{tabular}{@{}lcccccc@{}}
\toprule
Config & CHAIR$_s\downarrow$ & CHAIR$_i\downarrow$ & BERT$\uparrow$ & Calls & Wall-clock & Avg.\ rounds \\
\midrule
Frozen                 & .070 & .069 & .588 & 1   & 38.5\,s  & --   \\
\method{} $T\!=\!3$    & .055 & .040 & \textbf{.651} & 8 & 133\,s & 3 \\
\method{} $T\!=\!5$    & \textbf{.050} & \textbf{.035} & .643 & 12 & 199.3\,s & 5 \\
\rowcolor{tigerblue}
\method{} + early stop & .060 & .040 & .637          & 7.4 & 123\,s   & 2.71 \\
\bottomrule
\end{tabular}}
\end{table}

\begin{figure}[t]
\centering
\includegraphics[width=\columnwidth]{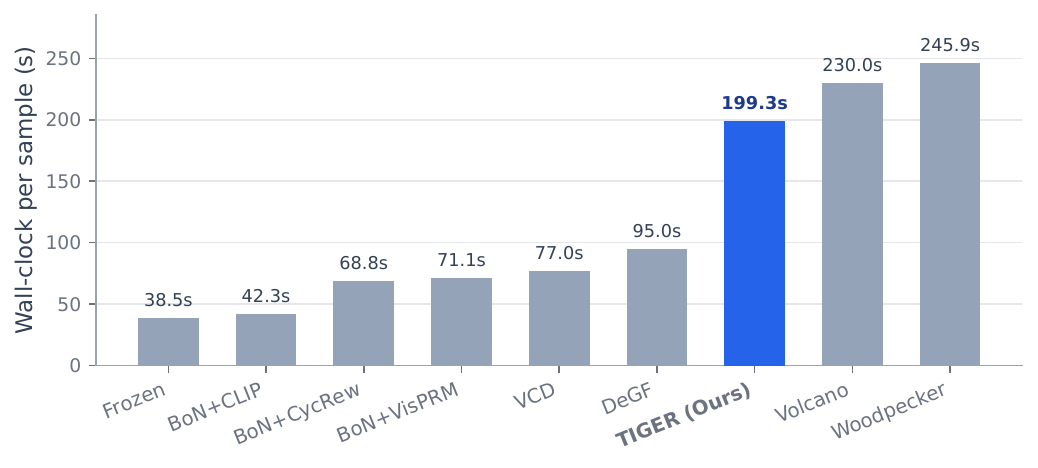}
\caption{Wall-clock seconds per COCO sample on a single GPU
(Qwen2.5-Omni-7B backbone). Bars ordered by ascending cost.}
\label{fig:compute_cost}
\end{figure}


\section{Conclusion}
\label{sec:conclusion}

We presented \method{}, an inference-time framework for fact-level
hallucination repair in multimodal generation. \method{} separates
input observation and output claim extraction, ranks claims by
deterministic support/conflict risk, and locally repairs high-risk
claims with a frozen backbone. Experiments across four cross-modal
paths and five open-source and proprietary backbones show that
\method{} reduces unsupported content while preserving task quality.
Ablations confirm that naive self-correction can amplify
hallucination and that iterative repair, atomic projection, and
risk-based ranking are jointly necessary.

\section*{Limitations}

This work focuses on four cross-modal generation paths and the
benchmarks considered in the paper; broader settings such as 3D
scenes, streaming multimodal inputs, and highly specialized domains
may require additional evaluation. \method{} represents evidence as
localized factual triples, which is less direct for abstract,
causal, or subjective content, although our annotation in
Appendix~\ref{app:coverage} shows that \method{} still improves over
direct decoding even when coverage is low. The repair loop also
depends on the quality of automatic fact extraction, but the
perturbation study in Appendix~\ref{app:perturb} shows graceful
degradation under moderate noise. Finally, iterative repair adds
inference-time computation; the lower-cost operating points in
Section~\ref{subsec:cost} reduce this overhead, and distilling
\method{}-repaired outputs into a lightweight model, for example
with learnability-grounded trajectory
selection~\cite{yu2026lark}, is a promising direction for future
work.



\bibliography{custom}

\appendix
\newpage
\section*{Appendix}
\addcontentsline{toc}{section}{Appendix}

\startcontents[sections]
\printcontents[sections]{l}{1}{\setcounter{tocdepth}{2}}

\section{Related Work}
\label{sec:related}

\noindent\textbf{Hallucination and verification.}
Unified multimodal models can process text, images, audio, and video,
but their textual outputs may contain facts that are not supported by
the input. This problem is closely related to cross-modal
inconsistency, where models produce conflicting claims for the same
content across
modalities~\cite{crossmodalconsistency2024,samecontent2025,sun2026learningtrustselectivecontext,dong2026generated}.
Multimodal hallucination has been studied through surveys,
benchmarks, and decoding
methods~\cite{hallsurvey2025,mmhalbench2023,wang2023amber,degf2025,lan-etal-2025-attention,fan2026chainsseeanswersdont}.
These studies provide useful evaluation tools, but many methods still
score or revise the whole response, which makes it difficult to
identify the specific facts that require correction. Structured
verification addresses this limitation by representing objects,
attributes, and relations explicitly. Scene graph, knowledge graph, and other graph based methods have
been used for fine grained evaluation, controllable generation,
knowledge grounding, and multimodal
reasoning~\cite{scenebench2024,generateanyscene2024,hyperglm2025,mario2026,liu2024graph,cai2025role,liu2024knowledge,zheng2024kg}.
\method{} follows this direction, but uses structured representations
for iterative repair. It extracts an observation graph from the input
and a claim graph from the current output, then compares each claim
against the input through support and conflict scores.

\noindent\textbf{Repair at inference time.}
Several methods reduce hallucination or improve generation quality
without updating the backbone, through inference-time interventions
such as query refinement, decoding control, and efficiency-oriented
model adaptation~\cite{zhou2026one,zhou2026look,zhou2025dropping}.
Resampling methods select one output from multiple candidates using
external signals such as visual alignment, process reward models, or
cycle consistency~\cite{visualprm2025,cyclereward2025}. These methods
can improve output quality, but they do not identify which facts in
the response are unsupported. Feedback based repair is closer to our
setting. Woodpecker~\cite{yin2024woodpecker},
Volcano~\cite{lee2024volcano}, and DeGF~\cite{degf2025} revise an
initial output by conditioning on both the input and the current
response. This design is simple, but the current response can affect
how the model interprets the input, so hallucinated claims may be
reinforced rather than removed. The feedback is also written in
natural language, which makes it difficult to rank facts or enforce a
repair budget. Related iterative, counter-example driven correction
has also been explored beyond multimodal generation, in
formal-language learning and agent
control~\cite{chen2024llms,chen2026towards,chen2025cedar}.
\method{} redesigns this step by extracting the input and output
independently, assigning each claim a graph-conditioned risk score,
and repairing only the selected high risk facts. This design makes
feedback quantifiable and reduces the dependence of repair on joint
generation.

\section{Prompts}
\label{app:prompts}

This appendix lists the verbatim prompt templates used by \method{} and
by the three text-feedback baselines. All prompt templates are
applicable to every backbone model evaluated in this paper (both
open-source and proprietary APIs) and do not depend on any specific
model architecture.

Notation follows Section~\ref{sec:method}. All methods share the same
generation step under a unified iterative repair framework. The
generation step produces the initial output via
$\mathcal{P}_{\text{gen}}$:
\begin{align*}
\mathbf{Y}_0 = \Phi(\mathcal{P}_{\text{gen}},\, \mathbf{X}).
\end{align*}
The methods differ in how the feedback signal $\mathcal{F}_t$ is
produced and how the repair prompt
$\mathcal{P}_{\text{refine}}^{(\cdot)}$ is instantiated. We describe
each in turn.

\paragraph{\method{}.}
\method{} decomposes feedback generation into two stages: independent
extraction and graph-conditioned risk computation. First, the input and
the current output are each independently projected into a structured
fact graph:
$G_{\mathbf{X}} = \Phi(\mathcal{P}_{\text{ext}},\, \mathbf{X})$,
$G_{\mathbf{Y}_t} = \Phi(\mathcal{P}_{\text{ext}},\, \mathbf{Y}_t)$;
$G_{\mathbf{X}}$ is extracted once before the repair loop and remains
fixed. The deterministic operator $\Psi_\alpha$ then computes per-fact
risk over the two graphs and selects the high-risk set as the feedback
signal: $\mathcal{F}_t = \Psi_\alpha(G_{\mathbf{X}},\,
G_{\mathbf{Y}_t})$. Finally, the repair prompt
$\mathcal{P}_{\text{refine}}$ produces the revised output from the raw
input, the current output, and the high-risk fact set:
\begin{align*}
    \mathbf{Y}_{t+1} = \Phi(\mathcal{P}_{\text{refine}},\, \mathbf{X},\, \mathbf{Y}_t,\, \mathcal{F}_t).
\end{align*}

Per sample, \method{} calls the backbone $\Phi$ a total of $2 + 2T$
times: $\mathcal{P}_{\text{gen}}$ once,
$\mathcal{P}_{\text{ext}}$ $1{+}T$ times (once for $\mathbf{X}$ and
once per round for each $\mathbf{Y}_t$), and
$\mathcal{P}_{\text{refine}}$ $T$ times.

\paragraph{Volcano~\citep{lee2024volcano}.}
We adapt Volcano's critique-then-revise protocol to the shared
frozen backbone. We do not use a separate Volcano checkpoint, so
the comparison remains backbone-matched: each round first produces
a natural-language critique conditioned on the input and the current
response, then revises the response conditioned on its own critique.
Formally, in round $t$:
\begin{align*}
    \mathcal{F}_t       &= \Phi(\mathcal{P}_{\text{fb}}^{\text{V}},\, \mathbf{X},\, \mathbf{Y}_t), \\
    \mathbf{Y}_{t+1}    &= \Phi(\mathcal{P}_{\text{refine}}^{\text{V}},\, \mathbf{Y}_t,\, \mathcal{F}_t).
\end{align*}
We adopt the critique and revise prompt style released in the
official Volcano repository (reproduced in
Appendix~\ref{app:prompts:baselines}) and use the paper's default
$T=3$ critique-revise rounds. Per sample Volcano calls the shared
backbone $\Phi$ a total of $1+2T$ times: the initial generation
once, then $T$ critique passes and $T$ revise passes. This
adaptation applies uniformly across all modalities (image, audio,
video); the procedure and prompts are unchanged, only the backbone
is shared with the other methods.

\paragraph{Woodpecker~\citep{yin2024woodpecker}.}
We faithfully reimplement Woodpecker's five-stage hallucination
correction pipeline as described in the original paper. Each stage
uses a dedicated prompt template adapted from the official Woodpecker
codebase; the visual-validation stage uses
Grounding~DINO~\citep{liu2024grounding} (already loaded for our
deterministic-extractor comparison in
Appendix~\ref{app:extractor_variants}) as the open-vocabulary visual
detector. The pipeline runs once per sample ($T=1$); formally:
\begin{align*}
    \mathbf{K}   &= \Phi(\mathcal{P}^{\text{W}}_{\text{kce}},\, \mathbf{Y}_0), \\
    \mathbf{Q}   &= \Phi(\mathcal{P}^{\text{W}}_{\text{qf}},\, \mathbf{K}), \\
    \mathbf{V}   &= \mathrm{DINO}(\mathbf{Q},\, \mathbf{X}), \\
    \mathbf{C}   &= \Phi(\mathcal{P}^{\text{W}}_{\text{vcg}},\, \mathbf{V}), \\
    \mathbf{Y}_T &= \Phi(\mathcal{P}^{\text{W}}_{\text{refine}},\, \mathbf{X},\, \mathbf{Y}_0,\, \mathbf{C}),
\end{align*}
where $\mathbf{K}$ is the list of key concepts extracted from
$\mathbf{Y}_0$, $\mathbf{Q}$ is the corresponding set of verification
questions, $\mathbf{V}$ is the Grounding-DINO output (boxes and
labels), and $\mathbf{C}$ is the visual-claim list that serves as
the explicit feedback signal $\mathcal{F}$ for the refine step.
The five stages decompose as follows:

\begin{enumerate}
\item \textbf{Key concept extraction.} A prompt extracts the list of
object-level concepts mentioned in $\mathbf{Y}_0$ that need
verification. Output: a JSON list of (entity, attribute) pairs.

\item \textbf{Question formulation.} For each extracted concept, a
prompt formulates a yes/no verification question of the form
``Does the image contain a \{concept\}?'' or
``Is the \{entity\} \{attribute\}?''.

\item \textbf{Visual knowledge validation.} For each verification
question, Grounding~DINO is invoked with a text query corresponding
to the concept (confidence threshold $0.35$, top-$5$ detections per
query). The detected bounding boxes and labels constitute the
``visual evidence'' for that question.

\item \textbf{Visual claim generation.} The detected evidence is
formatted into a structured claim list, e.g.\ ``GroundingDINO finds
2 boxes labelled \{label\} with scores \{...\}''. This list becomes
the explicit feedback $\mathcal{F}$ for the correction stage.

\item \textbf{Hallucination correction.} The backbone $\Phi$ is
called once more with the image, the original response, and the
visual claim list to produce the corrected response $\mathbf{Y}_T$.
\end{enumerate}

\noindent Per sample Woodpecker calls $\Phi$ four times
($\mathcal{P}_{\text{gen}}$, key-concept extraction, question
formulation, hallucination correction) plus one Grounding~DINO call
per extracted concept. The five-stage pipeline runs once per sample
rather than iteratively, matching the original Woodpecker protocol.
For modalities outside Grounding~DINO's coverage (audio, video),
Woodpecker is not applicable and is therefore omitted from the
audio/video benchmarks (Tables~\ref{tab:main_audio},
\ref{tab:main_video}). Prompts for the five stages are reproduced
verbatim in Appendix~\ref{app:prompts:baselines}.

\paragraph{DeGF~\citep{degf2025}.}
DeGF (Self-Correcting Decoding with Generative Feedback) uses a
text-to-image diffusion model to generate an auxiliary visual
reference from the model's initial response, then performs
contrastive decoding between the original image and the generated
image. We follow the official protocol with two adjustments noted
below. Unlike Volcano and \method{}, DeGF runs a single repair pass
($T=1$) and consumes the feedback at the logit level rather than the
prompt level:
\begin{align*}
\mathbf{Y}_0           &= \Phi(\mathcal{P}_{\text{gen}},\, \mathbf{X}), \\
\mathbf{X}_{\text{aux}} &= \mathrm{SD}(\mathbf{Y}_0), \\
\mathbf{Y}_T           &= \Phi^{\text{CD}}\bigl(\mathcal{P}_{\text{gen}},\, \mathbf{X},\, \mathbf{X}_{\text{aux}};\, \alpha\bigr),
\end{align*}
where $\mathrm{SD}(\cdot)$ is a Stable-Diffusion text-to-image
generator and $\Phi^{\text{CD}}$ decodes greedily from contrastive
logits per DeGF:
\[
\tilde{s}^{(k)} = (1+\alpha)\,\Phi(\mathcal{P}_{\text{gen}}, \mathbf{X})^{(k)} - \alpha\,\Phi(\mathcal{P}_{\text{gen}}, \mathbf{X}_{\text{aux}})^{(k)},
\]
at every decoding step $k$. 

The pipeline for each sample is:

\begin{enumerate}
\item \textbf{Initial generation.} Produce
$\mathbf{Y}_0 = \Phi(\mathcal{P}_{\text{gen}}, \mathbf{X})$ as usual,
where $\mathbf{X}$ is the original image $I_{\text{orig}}$.

\item \textbf{Generative feedback via diffusion.} Take the
caption-like response $\mathbf{Y}_0$ as a prompt and synthesise an
auxiliary image $I_{\text{gen}} = \mathrm{SD}(\mathbf{Y}_0)$ using
Stable Diffusion. We use SD-Turbo
($\mathrm{stabilityai/sd\text{-}turbo}$) with $1$ denoising step
and the default scheduler, matching the original paper's
``efficient'' configuration.

\item \textbf{Contrastive decoding.} Run two forward passes of the
backbone with the same text prompt: one conditioned on
$I_{\text{orig}}$ producing logits
$s_{\text{orig}}^{(k)}$ at decoding step $k$, the other conditioned on
$I_{\text{gen}}$ producing $s_{\text{gen}}^{(k)}$. The corrected
output $\mathbf{Y}_T$ is decoded greedily from the contrastive
logits
\[
\tilde{s}^{(k)} = (1+\alpha)\, s_{\text{orig}}^{(k)} \;-\; \alpha\, s_{\text{gen}}^{(k)}, \qquad \alpha = 0.5,
\]
matching DeGF and the same $\alpha$ as the original paper.
\end{enumerate}

\noindent Per sample DeGF runs one diffusion call plus
$1+2L$ backbone forward passes (where $L$ is the number of decoded
tokens), so it is comparable in cost to VCD-style contrastive
decoding on the same backbone. The two protocol adjustments are:
(i) we use SD-Turbo (1-step) instead of the original SD~1.5 (50-step)
for tractability across $1{,}000$ images, and (ii) we use the same
frozen Qwen2.5-Omni-7B / LLaVA-1.5-7B backbones as every other
baseline rather than LLaVA-1.5-13B; both deviations are documented
in our code release. 

\paragraph{Key difference from \method{}.}
These baselines use response-conditioned correction signals, but
they do not independently extract an observation graph from the
input and a claim graph from the current output, nor do they
compute deterministic per-claim support/conflict risk for budgeted
repair. \method{} differs by making the feedback signal explicit,
fact-level, and rankable.

\subsection{\method{} Prompts}
\label{app:prompts:tiger}

\paragraph{$\mathcal{P}_{\text{gen}}$.}
\method{} does not rewrite the task prompt at the first round; the raw
dataset prompt is fed to $\Phi$ verbatim. The actual prompts used in
our experiments are listed below. For MMHal-Bench, Clotho, and VideoHallucer, the prompt varies per instance
according to the benchmark default; we show the template form.

\begin{prompt}[title={$\mathcal{P}_{\text{gen}}$ --- COCO captions}]
Please describe this image in detail.
\end{prompt}

\begin{prompt}[title={$\mathcal{P}_{\text{gen}}$ --- MMHal-Bench}]
\{per-instance question\}
\end{prompt}

\begin{prompt}[title={$\mathcal{P}_{\text{gen}}$ --- Clotho}]
\{per-instance audio question\}
\end{prompt}

\begin{prompt}[title={$\mathcal{P}_{\text{gen}}$ --- VideoHallucer}]
\{per-instance video question\}
\end{prompt}
\paragraph{$\mathcal{P}_{\text{ext}}$.}
A single extraction template instantiated under modality conditioning,
as described in Section~\ref{sec:redesign}. The multimodal form is
applied to $\mathbf{X}$ to produce $G_{\mathbf{X}}$; the text form is
applied to $\mathbf{Y}_t$ to produce $G_{\mathbf{Y}_t}$. Both forms
share an identical schema (field roles, standardized predicate
vocabulary, and triple output format) and differ only in their
few-shot exemplars and the source-domain phrasing.

\begin{prompt}[title={$\mathcal{P}_{\text{ext}}$ --- multimodal form,
applied to $\mathbf{X}$}]
You are a multimodal fact extraction system. Extract ALL observable facts from the provided input across all modalities (image, video, audio, text). Each fact must be a triple (subject, predicate, object).

STRICT OUTPUT FORMAT --- numbered list, one triple per line, parentheses required:

\quad 1. (subject, predicate, object)

\quad 2. (subject, predicate, object)

\quad \ldots

FIELD ROLES:

\quad subject: head entity, written as a bare noun phrase WITHOUT attributes. GOOD: (car, is, red). BAD: (red car, is, parked).

\quad predicate: the relation. Use the standardized forms below when applicable.

\quad object: tail entity, attribute value, or count.

STANDARDIZED PREDICATES:

\quad Attributes (color, size, material, shape): \texttt{is}. e.g., (car, is, red).

\quad Counts: \texttt{count}. e.g., (lights, count, 3).

\quad Spatial (directional; subject is the figure, object is the ground): \texttt{on / under / above / below / left of / right of / in front of / behind / inside}.

\quad Possession or wearing or carrying: \texttt{holding / wearing / carrying / riding}.

\quad Existence: \texttt{exists in}. e.g., (dog, exists in, image).

\quad Actions: bare verb (\texttt{jogging / cooking / talking}).

Example --- given an image of a park with a dog and a man jogging:

\quad 1. (dog, exists in, image)

\quad 2. (dog, is, golden)

\quad 3. (man, exists in, image)

\quad 4. (man, jogging on, path)

\quad 5. (man, wearing, red shirt)

\quad 6. (dog, left of, man)

\quad 7. (trees, count, 5)

\quad 8. (sky, is, blue)

Example --- given an audio clip of a busy street:

\quad 1. (cars, honking, loudly)

\quad 2. (people, talking, nearby)

\quad 3. (engine, running, idle)

\quad 4. (music, playing from, shop)

Example --- given text \textquoteleft The president announced a new policy on Tuesday\textquoteright:

\quad 1. (president, announced, new policy)

\quad 2. (announcement, happened on, Tuesday)

Rules:

\quad Extract every object, attribute, spatial relation, action, and count you can verify.

\quad One fact per triple. Do NOT combine multiple claims into one line.

\quad Always put the bare entity in subject; never bundle adjectives into subject.

\quad For spatial predicates, subject is the figure positioned relative to object.

\quad Aim for 8--20 triples. Be thorough but only include facts grounded in the input.

\quad Output ONLY the numbered triple list. No explanation, no headers, no prose.
\end{prompt}

\begin{prompt}[title={$\mathcal{P}_{\text{ext}}$ --- text form,
applied to $\mathbf{Y}_t$}]
You are a fact extraction system. Read the text below and decompose it into ALL factual claims as structured triples (subject, predicate, object).

STRICT OUTPUT FORMAT --- numbered list, one triple per line, parentheses required:

\quad 1. (subject, predicate, object)

\quad 2. (subject, predicate, object)

\quad \ldots

FIELD ROLES:

\quad subject: head entity, written as a bare noun WITHOUT adjectives. GOOD: (car, is, red). BAD: (red car, is, parked).

\quad predicate: the relation. Prefer the standardized forms below.

\quad object: tail entity, attribute value, or count.

STANDARDIZED PREDICATES:

\quad Attributes: \texttt{is}. e.g., (car, is, red).

\quad Counts: \texttt{count}. e.g., (lights, count, 3).

\quad Existence: \texttt{exists in}. e.g., (dog, exists in, image).

\quad Spatial (directional; subject is figure, object is ground): \texttt{on / under / above / below / left of / right of / in front of / behind / inside}.

\quad Possession / action: \texttt{holding / wearing / carrying / riding / using}.

\quad Free actions: bare verb (\texttt{cooking, jogging, talking, \ldots}).

Example 1 --- input: \textquoteleft A red car is parked near a tall building on a sunny day.\textquoteright

\quad 1. (car, is, red)

\quad 2. (car, parked near, building)

\quad 3. (building, is, tall)

\quad 4. (day, is, sunny)

Example 2 --- input: \textquoteleft The fire hydrant cap is yellow.\textquoteright

\quad 1. (fire hydrant cap, is, yellow)

Example 3 --- input: \textquoteleft There are three traffic lights in the image.\textquoteright

\quad 1. (traffic lights, exists in, image)

\quad 2. (traffic lights, count, 3)

Example 4 --- input: \textquoteleft A man in a white shirt is cooking in the kitchen while holding a knife.\textquoteright

\quad 1. (man, exists in, image)

\quad 2. (man, wearing, white shirt)

\quad 3. (man, cooking in, kitchen)

\quad 4. (man, holding, knife)

Example 5 --- input: \textquoteleft The dog is on top of the car.\textquoteright

\quad 1. (dog, on, car)

Rules:

\quad Extract every claim, even from a short single sentence.

\quad Keep the subject a bare noun; never bundle attributes into subject.

\quad Each claim = one triple. Do NOT combine multiple facts.

\quad For spatial predicates, subject is the figure positioned relative to object.

\quad Output ONLY the numbered triple list. No prose.

Now extract from this text:

\quad \{\textless input text \textgreater\}
\end{prompt}

\paragraph{$\mathcal{P}_{\text{refine}}$.}
The high-risk set $\mathcal{F}_t = \Psi_\alpha(G_{\mathbf{X}},
G_{\mathbf{Y}_t})$ is rendered as a natural-language list and
inserted into the template below. Following the same deletion-permissive
editing policy used for the text-feedback baselines in
\S\ref{app:prompts:baselines}, the prompt allows removal only when a
flagged claim cannot be verified and no grounded replacement is
available. The key distinction is that \method{} applies this policy
only to the risk-ranked claims selected by $\Psi_\alpha$.

\begin{prompt}[title={$\mathcal{P}_{\text{refine}}$ --- repair
$\mathbf{Y}_t \to \mathbf{Y}_{t+1}$}]
Below is your previous response, followed by a list of claims that have been flagged as HIGH RISK under the current evidence. A claim is flagged when it has weak support, possible conflict, or cross-modal disagreement with the input.

Your goal: produce a revised response with high factual reliability while preserving all verified content. Only re-examine the flagged claims listed below. For each flagged claim:

\quad (a) If the claim CONTRADICTS the input, correct it to match the input.

\quad (b) If the claim is directly supported by the input, keep it.

\quad (c) If the claim cannot be verified and no grounded replacement is available, remove it from the response.

Do not add new factual details. Do not modify unflagged claims except for minimal edits needed to keep the response coherent. A shorter response is acceptable only when unsupported flagged claims are removed. Preserve the tone and the overall answer to the task prompt.

Previous response:

\quad \textquotedblleft \{original\_text\} \textquotedblright

High-risk claims to re-examine:

\quad \{facts\_block\}

Revised response:
\end{prompt}

The slot \texttt{\{original\_text\}} is filled with $\mathbf{Y}_t$;
\texttt{\{facts\_block\}} is filled with the verbalized
$\mathcal{F}_t$.

\subsection{Baseline Prompts}
\label{app:prompts:baselines}

The calling conventions of the three text-feedback baselines are
described at the beginning of this appendix. Below we list the
concrete prompt templates for each baseline, annotated with their
symbolic roles.

All baseline numbers in Tables~\ref{tab:main_image},
\ref{tab:main_audio}, and~\ref{tab:main_video} are produced by our
own runs under a shared frozen-backbone protocol --- they are not
copied from the original papers, whose datasets, metrics, and
backbones differ from ours. Within this protocol we strive for the
strongest faithful reproduction available:
(i)~\textbf{Volcano} adapts the official critique-then-revise
protocol with the released prompt templates, run on the same shared
backbone as the other methods so the comparison is
backbone-matched (we do not use a separate Volcano checkpoint);
(ii)~\textbf{Woodpecker} runs the full five-stage pipeline of
Yin~et~al.~\cite{yin2024woodpecker}, with Grounding~DINO as the
open-vocabulary detector for the visual-validation stage and the
same backbone $\Phi$ for the LLM stages;
(iii)~\textbf{DeGF} uses Stable Diffusion (SD-Turbo) for generative
feedback and contrastive decoding per DeGF Eq.~(4), differing from
the original only by the SD variant (Turbo vs.\ 1.5) and by the
backbone choice required to match our shared protocol;
(iv)~\textbf{VCD}, \textbf{AAD}, and \textbf{TCD} implement the
published contrastive-decoding formula
$\tilde{s} = (1+\alpha)\, s_{\text{with}} - \alpha\, s_{\text{without}}$
verbatim, with modality-specific neutralization
(Gaussian-noised image / silent audio / time-shuffled video frames)
taken from the corresponding original recipes;
(v)~\textbf{BoN+CLIP}, \textbf{BoN+VisualPRM}, and
\textbf{BoN+CycleReward} use the officially released reward/scoring
checkpoints as the BoN reranker. Where a baseline's original
pipeline depends on modules outside our shared protocol (e.g.\
modality-specific tools for audio or video), we note the adaptation
inline; otherwise the implementation reproduces the original method
on the shared backbone.

\begin{prompt}[title={Volcano --- critique prompt
$\mathcal{P}_{\text{fb}}^{\text{V}}$ (verbatim from the official Volcano repository)~\citep{lee2024volcano}}]
You are given an image and a candidate response describing it.
Critique the response for any visual hallucinations, factual errors,
or unsupported claims. Be specific: for each problem, name the
object, attribute, count, or relation that is incorrect and explain
why. Do not rewrite the response; only provide the critique.

Image:

\quad \{image\}

Candidate response:

\quad \{response\_text\}

Critique:
\end{prompt}

\begin{prompt}[title={Volcano --- revise prompt
$\mathcal{P}_{\text{refine}}^{\text{V}}$ (verbatim from the official Volcano repository)~\citep{lee2024volcano}}]
You previously generated the candidate response below and received
the following critique. Revise the response to address the critique
while keeping all correctly described content. Output only the
revised response.

Image:

\quad \{image\}

Candidate response:

\quad \{response\_text\}

Critique:

\quad \{critique\_text\}

Revised response:
\end{prompt}

\begin{prompt}[title={Woodpecker stage 1 --- key concept extraction
$\mathcal{P}_{\text{kce}}^{\text{W}}$~\citep{yin2024woodpecker}}]
Read the following response and list the visual concepts (objects,
their attributes, counts, and pairwise relations) that it asserts
about the image. Output as a JSON list of items with fields
``entity'', ``attribute'', ``count''. Do not infer; copy only what
the response explicitly claims.

Response:

\quad \{response\_text\}

JSON list:
\end{prompt}

\begin{prompt}[title={Woodpecker stage 2 --- question formulation
$\mathcal{P}_{\text{qf}}^{\text{W}}$~\citep{yin2024woodpecker}}]
For each item in the JSON list below, formulate one short yes/no
verification question of the form ``Does the image contain a
\{entity\}?'' or ``Is the \{entity\} \{attribute\}?'' or
``How many \{entity\} are in the image?''. Output as a JSON list of
\{``question'': ..., ``query'': ...\} where ``query'' is the
open-vocabulary noun phrase to detect (e.g.\ ``red car'').

JSON list of concepts:

\quad \{concepts\_json\}

JSON list of questions:
\end{prompt}

\begin{prompt}[title={Woodpecker stage 3 --- visual validation
(Grounding~DINO call, no LLM prompt)~\citep{yin2024woodpecker}}]
For each \{``query'': $q$\} in the question list, call
Grounding~DINO with text query $q$ at confidence threshold $0.35$.
Record the top-$5$ detected bounding boxes and their confidence
scores. Skip an item if Grounding~DINO returns no detection.
\end{prompt}

\begin{prompt}[title={Woodpecker stage 4 --- visual claim generation
$\mathcal{P}_{\text{vcg}}^{\text{W}}$~\citep{yin2024woodpecker}}]
Given the visual detections below, write a short factual statement
about what the image contains. Use the detected counts and labels
literally; do not introduce concepts that were not detected. Output
as plain text, one sentence per detected concept.

Detections:

\quad \{dino\_detections\}

Visual claims:
\end{prompt}

\begin{prompt}[title={Woodpecker stage 5 --- hallucination correction
$\mathcal{P}_{\text{refine}}^{\text{W}}$~\citep{yin2024woodpecker}}]
You previously generated the response below for the given image. A
visual grounding tool has produced the following factual claims
about what the image actually contains. Rewrite the response so that
it agrees with the visual claims: correct any contradictions, drop
unsupported assertions, and keep the original phrasing wherever it
already matches the visual claims.

Image:

\quad \{image\}

Original response:

\quad \{response\_text\}

Visual claims from grounding tool:

\quad \{visual\_claims\}

Corrected response:
\end{prompt}

\begin{prompt}[title={DeGF --- diffusion feedback (Stable Diffusion call,
no LLM prompt)~\citep{degf2025}}]
Use the initial response $\mathbf{Y}_0$ as the text prompt to
Stable Diffusion (SD-Turbo, $1$ denoising step, default scheduler,
guidance scale $1.0$) and synthesise an auxiliary image
$I_{\text{gen}}$ of resolution $512\times 512$. This auxiliary image
serves as the generative-feedback visual reference for the
contrastive-decoding step.
\end{prompt}

\begin{prompt}[title={DeGF --- contrastive decoding
(no LLM prompt; logit-level operation)~\citep{degf2025}}]
At each decoding step $k$, compute two backbone forward passes with
the same text prompt: one conditioned on the original image
$I_{\text{orig}}$ giving logits $s_{\text{orig}}^{(k)}$, and one
conditioned on the diffusion-generated image $I_{\text{gen}}$ giving
$s_{\text{gen}}^{(k)}$. Decode greedily from the contrastive
logits
\[
\tilde{s}^{(k)} = (1+\alpha)\, s_{\text{orig}}^{(k)} \;-\; \alpha\, s_{\text{gen}}^{(k)},
\qquad \alpha = 0.5,
\]
matching DeGF Eq.~(4). No further LLM prompt is used; the corrected
response $\mathbf{Y}_T$ is the greedy decode under
$\tilde{s}^{(k)}$.
\end{prompt}
\section{Risk Function Design and Property Verification}
\label{app:risk}

This appendix presents the full design of per-fact support $s(f)$,
conflict $c(f)$, and risk $r(f)$ used in \method{}
(\S\ref{app:risk:design}), and verifies that the risk function
satisfies three natural properties (\S\ref{app:risk:properties}).

\subsection{Support, Conflict, and Risk Design}
\label{app:risk:design}

\paragraph{Preliminaries.}
Every fact in the observation graph $G_{\mathbf{X}}$ and the claim
graph $G_{\mathbf{Y}_t}$ is represented as a triple
$f = (s_f, p_f, o_f)$ of subject, predicate, and object text fields,
following the scene-graph tuple structure of
SPICE~\citep{anderson2016spice}. The two graphs are not loose triple
sets: facts that share the same entity are connected by coreference
edges. For example, $(\text{man}, \text{wearing}, \text{red shirt})$
and $(\text{man}, \text{holding}, \text{coffee})$ share the subject
``man'' and are therefore linked.

We encode each text field with a frozen sentence transformer
$\Phi_{\text{enc}}$
(Sentence-BERT;~\citealp{reimers2019sentencebert}; checkpoint
\texttt{all-MiniLM-L6-v2}) into an L2-normalized vector. We convert
the raw cosine similarity into a clipped cosine similarity
\begin{equation}
\label{eq:clipped_cosine}
\kappa(t_1, t_2)
=
\max\{0,\Phi_{\text{enc}}(t_1)^\top \Phi_{\text{enc}}(t_2)\}
\in [0,1],
\end{equation}
where $\|\Phi_{\text{enc}}(\cdot)\|_2 = 1$. All subsequent field
similarities use $\kappa(\cdot,\cdot)$ instead of the raw cosine.

Given two facts $f = (s_f, p_f, o_f)$ and $g = (s_g, p_g, o_g)$, let
$\sigma_s = \kappa(s_f, s_g)$, $\sigma_p = \kappa(p_f, p_g)$, and
$\sigma_o = \kappa(o_f, o_g)$ denote the per-field bounded
similarities. The similarity between $f$ and $g$ is defined as the
equal-weight mean of the three fields:
\begin{equation}
\label{eq:sim}
\mathrm{sim}(f, g) \;=\; \tfrac{1}{3}(\sigma_s + \sigma_p + \sigma_o).
\end{equation}
Since each field similarity lies in $[0,1]$, we have
$\mathrm{sim}(f,g) \in [0,1]$. Computing similarity per field rather
than encoding the concatenated triple preserves structural information
and avoids the problem that whole-triple encoding assigns nearly
identical similarity to ``man riding horse'' and ``horse riding man.''

\paragraph{Support $s(f)$.}
Support measures how strongly a claim is backed by evidence in the
observation graph. The computation has two steps.

First, the \emph{local support} of each claim
$f \in G_{\mathbf{Y}_t}$ is the maximum similarity over all facts in
the observation graph:
\begin{equation}
\label{eq:local_support}
s_0(f) \;=\; \max_{g \in G_{\mathbf{X}}}\; \mathrm{sim}(f, g).
\end{equation}
The maximum is taken because a claim needs only one matching fact in
$G_{\mathbf{X}}$ to be supported. Because
$\mathrm{sim}(f,g) \in [0,1]$, we also have $s_0(f) \in [0,1]$.

Second, local support may be too low when the extractor misses a
fact. For instance, if the input image shows a man wearing a red shirt
and holding coffee but the extractor only recovers
$(\text{man}, \text{holding}, \text{coffee})$, then the claim
``wearing red shirt'' receives a low $s_0$ and would be incorrectly
flagged as high risk. To compensate for extraction omissions, we
propagate local support along coreference edges in the claim graph
$G_{\mathbf{Y}_t}$:
\begin{equation}
\label{eq:graph_support}
s(f) \;=\; \max_{f' \in \{f\} \cup \mathcal{N}_K(f)}\;
\gamma^{\,d(f,\,f')} \cdot s_0(f'),
\end{equation}
where $\mathcal{N}_K(f)$ is the $K$-hop coreference neighborhood of
$f$ in $G_{\mathbf{Y}_t}$ (BFS expansion), $d(f, f')$ is the hop
distance, and $\gamma \in (0, 1)$ is a decay factor. The geometric
decay $\gamma^d$ ensures that distant neighbors contribute
progressively less, so that indirect support remains local to the
entity neighborhood. In the example above, ``wearing red shirt'' is
linked to ``holding coffee'' through the shared entity ``man'' and
therefore receives indirect support
$\gamma \cdot s_0(\text{holding coffee})$. We set $K = 3$ and
$\gamma = 0.7$ throughout all experiments. When $K = 0$, the
propagation is disabled and $s(f)$ reduces to $s_0(f)$. Since
$s_0(f') \in [0,1]$ and $\gamma^{d(f,f')} \in [0,1]$, the propagated
support also satisfies $s(f) \in [0,1]$.

Graph propagation is what distinguishes \method{} from plain
triple-set comparison: without graph structure every fact is evaluated
independently; with propagation, facts about the same entity reinforce
each other and reduce false positives caused by extraction omissions.

\paragraph{Conflict $c(f)$.}
Conflict measures the strongest contradiction between a claim and the
evidence in the observation graph. A conflict arises when two facts
discuss the same topic (matching subject and predicate) but reach
different conclusions (different object), corresponding to the
definition of \emph{contradiction} in natural language inference: a
contradiction requires the same premise but an opposing conclusion, as
opposed to \emph{neutral} (unrelated topics).

Given $f$ and $g$, conflict is defined as the product of topic
consistency and conclusion divergence:
\begin{equation}
\label{eq:conflict}
\mathrm{conflict}(f, g) \;=\;
\underbrace{\tfrac{1}{2}(\sigma_s + \sigma_p)}_{\text{topic consistency}}
\;\cdot\;
\underbrace{(1 - \sigma_o)}_{\text{conclusion divergence}}.
\end{equation}
The multiplicative structure implements a natural soft gate. When the
subjects or predicates do not match, the left factor is close to~0 and
the conflict is suppressed regardless of how different the objects
are, because the two facts do not discuss the same topic. When the
subjects and predicates match but the objects also match, the right
factor is close to~0 and there is no conflict because the conclusions
agree. Only when the same entity and relation are paired with
different conclusions do both factors take high values and produce a
significant conflict. This soft gate introduces no threshold
hyperparameters.

The conflict score of a claim $f$ is the maximum conflict over all
facts in the observation graph:
\begin{equation}
\label{eq:conflict_score}
c(f) \;=\; \max_{g \in G_{\mathbf{X}}}\; \mathrm{conflict}(f, g).
\end{equation}
Unlike support, conflict is \emph{not} propagated along coreference
edges. Conflict is specific to an individual claim and does not
transfer to neighbors: if $(\text{man}, \text{is}, \text{tall})$
conflicts with the observation $(\text{man}, \text{is}, \text{short})$,
the neighboring fact $(\text{man}, \text{wearing}, \text{hat})$ is
unaffected because hat and height are unrelated. Propagating conflict
would let one erroneous fact contaminate all facts about the same
entity.

Since $\sigma_s,\sigma_p,\sigma_o \in [0,1]$, topic consistency
$\tfrac{1}{2}(\sigma_s + \sigma_p) \in [0, 1]$ and conclusion
divergence $(1 - \sigma_o) \in [0, 1]$. Therefore,
$\mathrm{conflict}(f,g) \in [0,1]$, and the maximum over
$G_{\mathbf{X}}$ guarantees $c(f) \in [0, 1]$.

\paragraph{Risk $r(f)$.}
Combining support and conflict, the per-fact risk is defined as
\begin{equation}
\label{eq:risk}
r(f) \;=\; (1 - s(f)) + \lambda \cdot c(f), \qquad \lambda > 0,
\end{equation}
where $(1 - s(f))$ measures lack of support and $\lambda \cdot c(f)$
measures active conflict. Because $s(f)$ and $c(f)$ are computed by
deterministic scoring operations after the graphs are fixed,
$r(\cdot)$ always returns the same value for a fixed
$(G_{\mathbf{X}}, G_{\mathbf{Y}_t})$ input, satisfying the
deterministic scoring requirement of the operator $\Psi_\alpha$
(Section~\ref{sec:redesign}).

\subsection{Property Verification}
\label{app:risk:properties}

We verify that the risk function
$r(s, c) = (1 - s) + \lambda c$ ($\lambda > 0$, $s, c \in [0, 1]$)
satisfies three natural properties.

\begin{proposition}
\label{prop:risk_properties}
The risk function $r(s, c) = (1 - s) + \lambda c$ with $\lambda > 0$
satisfies the following properties:

\textnormal{\textbf{(P1)} Non-negativity:}
$r(s, c) \ge 0$ for all $(s, c) \in [0, 1]^2$.

\textnormal{\textbf{(P2)} Boundary condition:}
$r(s, c) = 0$ if and only if $s = 1$ and $c = 0$.

\textnormal{\textbf{(P3)} Monotonicity:}
$r$ is strictly decreasing in $s$ and strictly increasing in $c$.
\end{proposition}

\begin{proof}
\emph{(P1).}
From $s \in [0, 1]$ we have $1 - s \ge 0$; from $c \in [0, 1]$ and
$\lambda > 0$ we have $\lambda c \ge 0$. The sum of two non-negative
terms is non-negative:
$r(s, c) = (1 - s) + \lambda c \ge 0$.

\emph{(P2).}
Sufficiency: substituting $s = 1$, $c = 0$ gives
$r(1, 0) = 0 + 0 = 0$.
Necessity: suppose $r(s, c) = 0$, i.e.,
$(1 - s) + \lambda c = 0$. By~(P1) both terms are non-negative; the
sum of two non-negative numbers is zero if and only if both are zero.
$1 - s = 0$ gives $s = 1$; $\lambda c = 0$ with $\lambda > 0$ gives
$c = 0$.
Hence zero risk is equivalent to full support and no conflict: a fact
is risk-free only when it has a perfect match in the observation graph
and contradicts no observed fact.

\emph{(P3).}
$\partial r / \partial s = -1 < 0$, so $r$ is strictly decreasing in
$s$: higher support lowers risk.
$\partial r / \partial c = \lambda > 0$, so $r$ is strictly increasing
in $c$: higher conflict raises risk.
\end{proof}

\paragraph{Interpretation.}
P1 ensures that risk scores can be used for ranking and selection.
P2 gives a precise semantics to zero risk: full support and no
conflict is the only condition under which a fact is marked as safe.
P3 ensures that risk ranking is consistent with intuition: more
support lowers risk, more conflict raises it. $\lambda$ is the sole
hyperparameter and controls the relative penalty of contradiction
versus lack of support.
\section{Convergence Proof for Iterative Risk Reduction}
\label{app:proof}

This appendix proves Theorem~\ref{thm:convergence}, establishing the
geometric convergence of the expected total risk under
Algorithm~\ref{alg:repair}.

\subsection{Setup and Notation}

At round $t$, the claim graph $G_{\mathbf{Y}_t}$ contains
$N_t := |G_{\mathbf{Y}_t}|$ facts. The total measured risk is
\[
R^{(t)} \;=\; \sum_{f \in G_{\mathbf{Y}_t}} r^{(t)}(f),
\]
where $r^{(t)}(f)$ is the per-fact risk defined in
Eq.~\eqref{eq:risk}. The high-risk set $\mathcal{F}_t$ selected at
line~11 of Algorithm~\ref{alg:repair} contains the
$\lceil \alpha N_t \rceil$ facts with the largest $r(\cdot)$. The
total risk decomposes into the selected and retained portions:
\[
R^{(t)}
\;=\;
\underbrace{\sum_{f \in \mathcal{F}_t} r^{(t)}(f)}_{\text{to repair}}
\;+\;
\underbrace{\sum_{f \in G_{\mathbf{Y}_t} \setminus \mathcal{F}_t}
            r^{(t)}(f)}_{\text{retained}}.
\]

The repair step changes the risk of facts in $\mathcal{F}_t$ from
$r^{(t)}(f)$ to $\widetilde{r}^{(t+1)}(f)$ but may also affect facts
outside $\mathcal{F}_t$. We define the \emph{ideal} post-repair total
risk (assuming perfect extraction) as
\[
\widetilde{R}^{(t+1)}
=
\sum_{f \in \mathcal{F}_t} \widetilde{r}^{(t+1)}(f)
+\sum_{f \in G_{\mathbf{Y}_t} \setminus \mathcal{F}_t} r^{(t)}(f)
\;+\; \Delta_t^{+},
\]
where $\Delta_t^{+} \ge 0$ is the incremental risk that repairing
$\mathcal{F}_t$ introduces on the remaining facts. The gap between
the \emph{measured} risk on the actually extracted graph
$G_{\mathbf{Y}_{t+1}}$ and the ideal risk is
\[
\Delta_t^{\mathrm{ext}}
\;:=\; R^{(t+1)} - \widetilde{R}^{(t+1)}.
\]
Combining the two definitions gives the one-step decomposition:
\begin{align}
\label{eq:onestep}
    R^{(t+1)}
\;=\;&
\sum_{f \in \mathcal{F}_t} \widetilde{r}^{(t+1)}(f)
\;+\; \sum_{f \in G_{\mathbf{Y}_t} \setminus \mathcal{F}_t} r^{(t)}(f)\notag\\
&\;+\; \Delta_t^{+}
\;+\; \Delta_t^{\mathrm{ext}}.
\end{align}

\subsection{Assumptions}

\begin{assumption}[Bounded graph size]
\label{asm:size}
There exists a constant $N_{\max}$ such that $N_t \le N_{\max}$ for
all rounds~$t$.
\end{assumption}

\begin{assumption}[Per-fact repair progress]
\label{asm:progress}
There exists a constant $\varepsilon \in (0, 1]$ such that for every
$f \in \mathcal{F}_t$,
\[
\mathbb{E}\!\left[\widetilde{r}^{(t+1)}(f) \mid R^{(t)}\right]
\;\le\; (1 - \varepsilon)\, r^{(t)}(f).
\]
\end{assumption}

\begin{assumption}[Bounded side effects]
\label{asm:side}
There exists a constant $\beta \ge 0$ such that
$\mathbb{E}[\Delta_t^{+} \mid R^{(t)}] \le \beta$.
\end{assumption}

\begin{assumption}[Bounded extraction loss]
\label{asm:extract}
There exists a constant $\xi \ge 0$ such that
$\mathbb{E}[\Delta_t^{\mathrm{ext}} \mid R^{(t)}] \le \xi$.
\end{assumption}

\subsection{Proof of Theorem~\ref{thm:convergence}}

\begin{proof}
Taking conditional expectation of the one-step
decomposition~\eqref{eq:onestep} given $R^{(t)}$:
\begin{align}
&\mathbb{E}\!\left[R^{(t+1)} \mid R^{(t)}\right] \notag\\
\;=\;&
\sum_{f \in \mathcal{F}_t}\!
   \mathbb{E}\!\left[\widetilde{r}^{(t+1)}(f) \mid R^{(t)}\right]
+ \sum_{f \in G_{\mathbf{Y}_t} \setminus \mathcal{F}_t}\! r^{(t)}(f)\notag\\
&+ \mathbb{E}\!\left[\Delta_t^{+} \mid R^{(t)}\right]
+ \mathbb{E}\!\left[\Delta_t^{\mathrm{ext}} \mid R^{(t)}\right]
\notag\\
\;\overset{(a)}{\le}\;&
(1 - \varepsilon)\!\sum_{f \in \mathcal{F}_t}\! r^{(t)}(f)
\;+\; \sum_{f \in G_{\mathbf{Y}_t} \setminus \mathcal{F}_t}\! r^{(t)}(f)\notag\\
&\;+\; \beta + \xi
\notag\\
\;=\;&
R^{(t)} - \varepsilon \sum_{f \in \mathcal{F}_t} r^{(t)}(f)
+ \beta + \xi
\notag\\
\;\overset{(b)}{\le}\;&
(1 - \alpha\varepsilon)\, R^{(t)} + \beta + \xi,
\label{eq:drift}
\end{align}
where~(a) applies Assumption~\ref{asm:progress} to each
$f \in \mathcal{F}_t$ and sums, and applies
Assumptions~\ref{asm:side} and~\ref{asm:extract} to the side-effect
and extraction terms; (b) holds because $\mathcal{F}_t$ contains the
$\lceil \alpha N_t \rceil$ highest-risk facts, so their average risk
is at least the overall average $R^{(t)}/N_t$, giving
$\sum_{f \in \mathcal{F}_t} r^{(t)}(f) \ge
\lceil \alpha N_t \rceil \cdot R^{(t)}/N_t \ge \alpha\, R^{(t)}$.

Taking unconditional expectation of~\eqref{eq:drift} yields the
recurrence
$\mathbb{E}[R^{(t+1)}] \le (1 - \alpha\varepsilon)\,
\mathbb{E}[R^{(t)}] + \beta + \xi$.
Unrolling by induction, suppose the bound holds at round~$t$; then
\begin{align}
&\mathbb{E}\!\left[R^{(t+1)}\right] \notag\\
\;\overset{(c)}{\le}\;&
(1 - \alpha\varepsilon)\, \mathbb{E}[R^{(t)}] + \beta + \xi
\notag\\
\;\overset{(d)}{\le}\;&
(1 - \alpha\varepsilon)\!\left[
   (1 - \alpha\varepsilon)^{t} R^{(0)}
   + (\beta{+}\xi)\!\sum_{j=0}^{t-1}(1{-}\alpha\varepsilon)^{j}
\right]\notag \\
&+ \beta + \xi
\notag\\
\;=\;&
(1 - \alpha\varepsilon)^{t+1} R^{(0)}
+ (\beta + \xi)\!\sum_{j=0}^{t}(1 - \alpha\varepsilon)^{j},
\notag
\end{align}
where~(c) is the drift inequality~\eqref{eq:drift} and~(d) is the
inductive hypothesis. The base case $t = 0$ holds with equality.
Since $\alpha\varepsilon \in (0, 1]$,
$\sum_{j=0}^{T-1}(1{-}\alpha\varepsilon)^{j} \le
1/(\alpha\varepsilon)$.
Evaluating at $t = T$ gives
\[
\mathbb{E}\!\left[R^{(T)}\right]
\;\le\;
(1 - \alpha\varepsilon)^{T}\, R^{(0)}
\;+\;
\frac{\beta + \xi}{\alpha\varepsilon}.
\qedhere
\]
\end{proof}

\paragraph{Asymptotic behavior.}
For $\alpha\varepsilon \in (0, 1)$,
$(1 - \alpha\varepsilon)^{T} \to 0$ as $T \to \infty$, and the bound
converges to the residual $(\beta + \xi)/(\alpha\varepsilon)$, which
characterizes the capability boundary of the framework.

\section{Experimental Details}
\label{app:experimental-details}

\subsection{Datasets and Preprocessing}
\label{app:exp:datasets}

We evaluate on five benchmarks that together cover the four
cross-modal generation paths reported in the main paper:
image$\to$text (COCO, AMBER), image+text$\to$text (MMHal-Bench),
audio$\to$text (Clotho), and video$\to$text (VideoHallucer). We also
use one curated probe set (SCS-1000) for the spurious-correlation
analysis in Section~\ref{sec:motivation} and the
feedback-mention-rate experiment in Section~\ref{subsec:mechanism}.
Splits are the official splits as redistributed by their upstream
sources; no custom random splits are introduced. SCS-1000 is a
curated \emph{probe} set, not a train / dev / test split.
\begin{table*}[t]
\centering
\caption{Datasets and splits. ``Full size'' is the upstream benchmark
size; ``\# used'' is the sample count reported in the main paper.
For COCO and SCS-1000, CHAIR-style object presence uses COCO
\texttt{instances\_val2014.json} augmented with a synonym table.}
\label{tab:datasets}
\scriptsize
\setlength{\tabcolsep}{3pt}
\renewcommand{\arraystretch}{1.05}
\resizebox{\textwidth}{!}{%
\begin{tabular}{@{}lllrrl@{}}
\toprule
\textbf{Dataset} & \textbf{Source} & \textbf{Split} & \textbf{Full size} & \textbf{\# used} & \textbf{Task} \\
\midrule
COCO val2014  & official val2014 captions~\cite{chen2015microsoft} & val (subset) & 40{,}504 & 1{,}000 & img$\to$txt (caption) \\
AMBER         & official query sets~\cite{wang2023amber} & gen./disc./attr./rel. & 1{,}004--7{,}628 & 1{,}000/task & img$\to$txt (halluc.) \\
MMHal-Bench   & HF~\cite{mmhalbench2023} & test & 96 & 96 & img+txt$\to$txt \\
Clotho        & HF~\cite{drossos2020clotho} & test & 1{,}045 & 1{,}000 & audio$\to$txt (caption) \\
VideoHallucer & HF~\cite{wang2024videohallucer} & test & 900 pairs & 900 & video$\to$txt (QA) \\
SCS-1000      & curated COCO val2014, 9 pairs (Tab.~\ref{tab:scs_pairs}) & val2014 & 1{,}000 & 1{,}000 & img$\to$txt (spur.\ corr.) \\
\bottomrule
\end{tabular}%
}
\end{table*}

No further preprocessing is applied: images are loaded at native
resolution by the chosen backbone's vision tower, audio at the
dataset's native sample rate. License information for each dataset
follows the upstream repository.

\paragraph{SCS-1000 cue pairs.}
SCS-1000 is the curated probe set used in Section~\ref{sec:motivation}
(co-occurrence hallucination rate, Figure~\ref{fig:spurious}) and
in the feedback-mention-rate analysis in
Section~\ref{subsec:mechanism}. We select nine $(a, b)$ pairs of
COCO object categories that co-occur frequently in COCO captions,
then for each pair sample roughly $110$ val2014 images that contain
the cue object $a$ but \emph{not} the absent object $b$ ($1{,}000$
images in total). The cue is verified present and the absent object
verified absent against COCO \texttt{instances\_val2014.json}.
Table~\ref{tab:scs_pairs} lists the nine pairs.

\begin{table}[h]
\centering
\caption{The nine SCS-1000 cue pairs. Each row holds the scene cue
$a$ that is present in the image and the COCO category $b$ that is
verified absent from the image but frequently co-occurs with $a$ in
COCO captions. The co-occurrence hallucination rate (CHR) in
Figure~\ref{fig:spurious} is the fraction of generations that
mention $b$ when conditioned on a SCS-1000 image containing $a$.}
\label{tab:scs_pairs}
\small
\setlength{\tabcolsep}{6pt}
\begin{tabular}{@{}ll@{}}
\toprule
Cue $a$ (present) & Absent $b$ (frequent co-occurrent) \\
\midrule
beach        & umbrella \\
bed          & person   \\
dining table & cup      \\
grass        & dog      \\
kitchen      & knife    \\
road         & car      \\
sky          & airplane \\
snow         & skis     \\
water        & boat     \\
\bottomrule
\end{tabular}
\end{table}

\subsection{Models and Baselines}
\label{app:exp:models}

\paragraph{Primary backbone.}
The primary backbone $\Phi$ is Qwen2.5-Omni-7B
(\texttt{Qwen/Qwen2.5-Omni-7B})~\citep{qwen25omni2025}, which
accepts text, image, audio, and video inputs and produces text and
audio outputs in a single architecture. The same model serves as
both the generator of $\mathbf{Y}_t$ and the extractor that produces
$G_{\mathbf{X}}$ and $G_{\mathbf{Y}_t}$ via
Eq.~\eqref{eq:extract}.

\begin{table}[h]
\centering
\caption{Primary backbone configuration. The backbone is frozen
throughout; no parameter updates occur at any stage.}
\label{tab:primary-backbone}
\small
\begin{tabular}{@{}ll@{}}
\toprule
\textbf{Field} & \textbf{Value} \\
\midrule
Model                 & Qwen2.5-Omni-7B \\
HuggingFace ID        & \texttt{Qwen/Qwen2.5-Omni-7B} \\
Size                  & 7B (vision + audio + text) \\
Class                 & multimodal, instruct \\
Tokenizer             & bundled (no override) \\
dtype                 & \texttt{auto} (bf16 on Ampere+) \\
\texttt{device\_map}  & \texttt{auto} (single-GPU) \\
\texttt{attn\_impl.}  & \texttt{sdpa} \\
Vision / audio tower  & bundled, frozen \\
Image resolution      & native, no resize \\
System prompt         & Qwen2.5-Omni default \\
Max sequence length   & HF default ($\sim$32k context) \\
\bottomrule
\end{tabular}
\end{table}

\paragraph{Secondary backbones (API).}
Used for the cross-backbone generalization experiment in
Appendix~\ref{app:api_backbones} only. All API calls disable
temperature and rely on the endpoint default (none of these models
expose a deterministic decoding flag at the time of writing).

\begin{itemize}
\item GPT-5.5 (Azure OpenAI, model \texttt{gpt-5.5}).
\item Gemini 3.5 Flash (Google AI Studio, model
\texttt{gemini-3.5-flash}).
\item Claude Haiku 4.5 (Anthropic API, model
\texttt{claude-haiku-4-5}).
\end{itemize}

We additionally report results with LLaVA-1.5-7B~\citep{liu2024improved}
on the same image-path benchmarks as the Qwen primary backbone, also
in Appendix~\ref{app:api_backbones}. API access windows are listed
in Appendix~\ref{app:reproducibility}.

\paragraph{Baselines compared.}
All baselines are reimplementations on the same Qwen2.5-Omni-7B
backbone (and on LLaVA-1.5-7B where applicable), called through the
same dispatcher with a different \texttt{--mode} flag, so the data
path, sampling hyperparameters, and evaluation protocol are matched
across methods; the per-method inference budgets differ by design
(e.g., best-of-$N$ methods draw $N$ candidates per sample, iterative
methods run $T$ repair rounds, contrastive decoders run two forward
passes per generated token) and are reported separately in
Table~\ref{tab:baselines}.
Per-baseline prompt templates are listed in
Appendix~\ref{app:prompts:baselines}. We compare against ten
baselines grouped into three families: resampling
(BoN+CLIP, BoN+VisualPRM, BoN+CycleReward), modality-specific
contrastive decoding (VCD on image, AAD on audio, TCD on video),
and iterative refinement (Volcano, Woodpecker, DeGF). Frozen is the
no-repair lower bound.

\begin{table*}[t]
\centering
\caption{Baselines used in the main experiments. ``Paths'' lists the
modalities each baseline is evaluated on:
image~$=$~COCO/AMBER/MMHal-Bench, audio~$=$~Clotho,
video~$=$~VideoHallucer, all~$=$~every benchmark. Budget is in
backbone forward passes; contrastive methods (VCD, AAD, TCD) run two
passes per generated \emph{token}.}
\label{tab:baselines}
\scriptsize
\setlength{\tabcolsep}{4pt}
\renewcommand{\arraystretch}{1.05}
\resizebox{\textwidth}{!}{%
\begin{tabular}{@{}lllc@{}}
\toprule
\textbf{Baseline} & \textbf{Implementation} & \textbf{Budget} & \textbf{Paths} \\
\midrule
Frozen                                  & direct decode, no repair                    & 1 pass/sample                               & all \\
BoN+CLIP                                & best-of-$N$ rerank by CLIP                  & $N{=}3$, $T{=}0.7$                          & image \\
BoN+VisPRM~\citep{visualprm2025}        & best-of-$N$ rerank by VisualPRM             & $N{=}3$                                     & image \\
BoN+CycRew~\citep{cyclereward2025}      & best-of-$N$ by round-trip F1                & $N{=}3$                                     & all \\
BoN+CLAP                                & best-of-$N$ rerank by CLAP                  & $N{=}3$                                     & audio \\
BoN+VideoCLIP                           & best-of-$N$ rerank by VideoCLIP             & $N{=}3$                                     & video \\
VCD~\citep{leng2024mitigating}          & contrast vs.\ noised image                  & 2 pass/token                                & image \\
AAD~\citep{hsu2025reducing}             & contrast vs.\ silent audio                  & 2 pass/token                                & audio \\
TCD~\citep{zhang2024eventhallusion}     & contrast vs.\ shuffled frames               & 2 pass/token                                & video \\
Volcano~\citep{lee2024volcano}          & adapted critique+revise on shared backbone  & $1{+}2T$ calls/sample                       & all \\
Woodpecker~\citep{yin2024woodpecker}    & 5-stage pipeline w/ Grounding DINO          & 4 calls + DINO per concept                  & image \\
DeGF~\citep{degf2025}                   & SD-Turbo aux + contrastive decoding         & 1 diffusion + 2 pass/token                  & image \\
\midrule
\rowcolor{tigerblue}
\method{} (ours)                        & graph extraction + risk-ranked repair       & $2{+}2T$ pass/sample                        & all \\
\bottomrule
\end{tabular}%
}
\end{table*}

\subsection{Inference and Evaluation Protocol}
\label{app:exp:protocol}

\paragraph{Per-dataset hyperparameters.}
Table~\ref{tab:hparams} lists the operating point of TIGER on each
dataset. The operating points for $T$, $\alpha$, and $\lambda$ are listed in Table~\ref{tab:hparams}. 
The sensitivity study in Figure~\ref{fig:sensitivity} is conducted on COCO to assess robustness around the default setting. 
We fix the graph-propagation parameters to $K=3$ and $\gamma=0.7$ across all experiments.

\begin{table*}[t]
\centering
\caption{Per-dataset hyperparameters of \method{}. All five
benchmarks share the same Qwen2.5-Omni-7B backbone and the same
sampling decoder (temperature $=0.7$, top-$p=0.9$). Results are
reported as mean$\pm$std over three decoding seeds, as described in
Appendix~\ref{app:exp:stats}; dataset-specific operating points are
listed below. The two question-answering datasets (MMHal-Bench,
VideoHallucer) use a single-fact repair regime (one flagged claim per
round); the free-form datasets (COCO, AMBER, Clotho) use the
batch-repair regime ($\alpha\!=\!0.2$).}
\label{tab:hparams}
\footnotesize
\setlength{\tabcolsep}{8pt}
\renewcommand{\arraystretch}{1.15}
\begin{tabular}{@{}lccccc@{}}
\toprule
& \textbf{COCO} & \textbf{AMBER} & \textbf{MMHal} & \textbf{Clotho} & \textbf{VHallucer} \\
\midrule
Rounds $T$               & 5    & 5    & 3             & 5    & 3 \\
Budget $\alpha$          & 0.2  & 0.2  & single-fact   & 0.2  & single-fact \\
Conflict $\lambda$       & 0.5  & 0.5  & 1.5           & 0.5  & 1.5 \\
Coref $\gamma / K$       & 0.7\,/\,3 & 0.7\,/\,3 & 0.7\,/\,3 & 0.7\,/\,3 & 0.7\,/\,3 \\
\texttt{max\_new\_tokens} & 128 & 128  & 512           & 128  & 12 \\
Decoding                 & \multicolumn{5}{c}{sampling, temperature $0.7$, top-$p$ $0.9$ (shared across all five benchmarks)} \\
Seeds                    & \multicolumn{5}{c}{42 / 43 / 44 (results are mean$\pm$std over the three seeds)} \\
\bottomrule
\end{tabular}
\end{table*}

\paragraph{Chain-of-thought and answer extraction.}
No chain-of-thought prompting is used. Outputs are taken verbatim
from the backbone. Object-presence judgements for CHAIR are taken
from the COCO \texttt{instances\_val2014.json} annotations and the
AMBER object-set annotations, not from the model. For MMHal-Bench
we report sentence-transformer BERTScore against the
reference answer rather than the original GPT-judge protocol,
because BERTScore is reproducible without an external LLM call and
correlates with the GPT-judge score on this benchmark. For
VideoHallucer we apply a robust yes/no parser to the free-form
output (regex-based, treating hedged responses such as ``hard to
tell'' as ``no''); the parser is shared across all methods in
Table~\ref{tab:main_video} so the comparison is fair.

\paragraph{Evaluation metrics.}
All headline metrics reported in the main paper are computed
independently of \method{}'s evidence graph, so the reported gains
cannot be an artefact of \method{} optimizing its own internal
quantities. Table~\ref{tab:independent-metrics} lists every metric
used in the main paper alongside the dataset it scores and the
external tool that produces it.

\begin{table*}[t]
\centering
\caption{Independent evaluation metrics used in the
main paper. Every metric is computed by an external tool against
ground-truth annotations distributed with the benchmark, so no
\method{}-internal score (support $s$, conflict $c$, risk $r$, or the
$G_{\mathbf{X}}$ / $G_{\mathbf{Y}_t}$ graphs) enters the headline
numbers.}
\label{tab:independent-metrics}
\small
\setlength{\tabcolsep}{4pt}
\begin{tabular}{@{}lll@{}}
\toprule
\textbf{Metric} & \textbf{Dataset(s)} & \textbf{Tool / reference} \\
\midrule
CHAIR$_s$, CHAIR$_i$ $\downarrow$ & COCO              & COCO \texttt{instances\_val2014.json} + synonym table \\
BERTScore $\uparrow$              & COCO, MMHal-Bench & \texttt{bert\_score} pkg vs.\ COCO captions / MMHal refs \\
Disc.\ Acc $\uparrow$             & AMBER             & official discriminative split, exact-match yes/no \\
CHAIR$_g$ $\downarrow$            & AMBER (generative) & AMBER object set per image \\
RougeL $\uparrow$, BLEU $\uparrow$ & Clotho            & \texttt{rouge\_score} / \texttt{nltk} vs.\ 5 ref captions \\
CLAP $\uparrow$                   & Clotho            & \texttt{laion/larger\_clap\_general} audio-text cosine \\
AEHR $\downarrow$                 & Clotho            & PANNs CNN14 top-10 events vs.\ predicted mentions \\
HallucRate $\downarrow$           & VideoHallucer     & rate of ``yes'' on hallucinated questions \\
Acc$_\text{b}$, Acc$_\text{h}$, Paired $\uparrow$ & VideoHallucer & robust yes/no parser, paired protocol \\
\bottomrule
\end{tabular}
\end{table*}

For the spurious-correlation probe (SCS-1000, Figure~\ref{fig:spurious}
in Section~\ref{sec:motivation}) we additionally report the
Co-occurrence Hallucination Rate (CHR), defined as the fraction of
generations that mention an absent object $b$ when the cue object $a$
is present; this follows the standard CHAIR-style fraction over the
curated cue pairs.

\paragraph{Cost model, early stopping, and loop exit.}
\label{app:earlystop}
Per sample, \method{} calls the backbone $2+2T$ times:
$\mathcal{P}_{\text{ext}}$ once for $\mathbf{X}$ and once per round
for each $\mathbf{Y}_t$, $\mathcal{P}_{\text{gen}}$ once, and
$\mathcal{P}_{\text{refine}}$ once per round. With fixed $T$, this
is a constant-factor overhead over direct decoding and scales
linearly with the number of samples. The graph-scoring step uses a
frozen MiniLM encoder and is small relative to backbone decoding.
Serving techniques such as PagedAttention~\citep{kwon2023efficient}
and speculative decoding~\citep{leviathan2023fast} can further reduce
latency and are orthogonal to our contribution.

The early-stopping variant in Section~\ref{subsec:cost} monitors the
mean per-fact risk $\bar r^{(t)} = R^{(t)}/|G_{\mathbf{Y}_t}|$,
which is available from the scores that $\Psi_\alpha$ computes and
requires no additional backbone call. The loop terminates when the
per-round improvement $\bar r^{(t-1)} - \bar r^{(t)}$ falls below
$\delta = 0.02$ for two consecutive rounds (patience $=2$); an
improvement of at least $\delta$ resets the counter. The variant
runs at least one round and is capped at $T\!=\!5$, so the earliest
exit occurs after the second round. On COCO, early stopping triggers
on 97\% of samples with 2.71 rounds on average; on Clotho, it
triggers on 99\% of samples with 2.49 rounds on average (three
seeds).

Independently of this variant, all \method{} runs use one shared
exit condition: the repair loop terminates before round $T$ when the
maximum risk among the remaining facts falls below $0.3$, since
selecting facts below this level yields no further repair. This
condition is an implementation detail of Algorithm~\ref{alg:repair}
and applies uniformly to every configuration reported in this paper.

\subsection{Mechanism Analysis Methodology}
\label{app:exp:mechanism}

This subsection documents how the three panels of
Figure~\ref{fig:rq4} in Section~\ref{subsec:mechanism} are produced.
All three probes use the SCS-1000 image set (Table~\ref{tab:scs_pairs})
and the Qwen2.5-Omni-7B backbone at the COCO operating point in
Table~\ref{tab:hparams} ($T\!=\!5$, $\alpha\!=\!0.2$, $\lambda\!=\!0.5$,
$\gamma\!=\!0.7$, $K\!=\!3$).

\paragraph{Panel (a): Feedback mention rate.}
For each image in SCS-1000 the cue object $a$ is verified present and
the absent object $b$ is verified absent against COCO
\texttt{instances\_val2014.json}. We run the same $T\!=\!5$ repair
loop with three different feedback channels:
\textbf{L1}~(naive joint feedback) instantiates $\mathcal{F}_t$ via
the joint-conditioning prompt of Eq.~\eqref{eq:naive-feedback},
i.e., $\mathcal{F}_t = \Phi(\mathcal{P}_{\text{fb}}, \mathbf{X}, \mathbf{Y}_t)$;
\textbf{L2}~(text feedback) first extracts $G_{\mathbf{X}}$ and
$G_{\mathbf{Y}_t}$ independently via Eq.~\eqref{eq:extract} and then
verbalizes both graphs into the feedback prompt as natural-language
text, but does not perform the deterministic risk computation $\Psi_\alpha$;
\textbf{L3}~(TIGER) uses $\mathcal{F}_t = \Psi_\alpha(G_{\mathbf{X}}, G_{\mathbf{Y}_t})$
of Eq.~\eqref{eq:tiger-feedback} and verbalizes only the top-$\lceil\alpha N\rceil$
risk-ranked atomic claims. The feedback mention rate is the fraction
of samples whose feedback text contains a case-insensitive whole-word
match for the absent object $b$ (plus the synonym table used by the
CHAIR scorer). Error bars are 95\% Wilson confidence intervals over
the 1{,}000 binary outcomes. A lower rate means the feedback channel
is less prone to re-introducing the co-occurrence prior identified
in Section~\ref{sec:motivation}.

\paragraph{Panel (b): Fact composition of $G_{\mathbf{Y}}$.}
For each image we (i) extract $G_{\mathbf{X}}$ once from the image
via Eq.~\eqref{eq:extract}; (ii) extract a reference graph $G_{\mathbf{GT}}$
from the five COCO ground-truth captions per image using the
text form of $\mathcal{P}_{\text{ext}}$; (iii) extract $G_{\mathbf{Y}}$
from the model output (Frozen $\mathbf{Y}_0$ or \method{} $\mathbf{Y}_T$).
Each fact $f \in G_{\mathbf{Y}}$ is classified into one of three
disjoint bins by the sentence-transformer similarity defined in
Eq.~\eqref{eq:sim}:
\begin{itemize}
\item \textbf{correct $\in G_{\mathbf{X}}$}: $\max_{g\in G_{\mathbf{GT}}}\mathrm{sim}(f, g) \ge \tau$ \emph{and} $\max_{g\in G_{\mathbf{X}}}\mathrm{sim}(f, g) \ge \tau$;
\item \textbf{correct $\notin G_{\mathbf{X}}$}: $\max_{g\in G_{\mathbf{GT}}}\mathrm{sim}(f, g) \ge \tau$ \emph{and} $\max_{g\in G_{\mathbf{X}}}\mathrm{sim}(f, g) < \tau$;
\item \textbf{wrong}: $\max_{g\in G_{\mathbf{GT}}}\mathrm{sim}(f, g) < \tau$.
\end{itemize}
We use a fixed support threshold $\tau\!=\!0.55$, so the classification rule is consistent
with how the rest of the framework labels facts. The bars in
panel~(b) are per-sample means of the bin counts, averaged across
the 1{,}000 SCS-1000 images. The \emph{correct $\notin G_{\mathbf{X}}$}
slice (light green) is the key diagnostic: it counts facts that the
extractor missed but the refine step recovered by reading the raw
input $\mathbf{X}$ directly, which is the empirical signature of
the asymptotic floor analysis in Appendix~\ref{app:proof}.

\paragraph{Panel (c): Sample-level similarity to $G_{\mathbf{GT}}$.}
For each sample $i$ and source graph $G \in \{G_{\mathbf{X}}, G_{\mathbf{Y}_T}\}$,
we compute the mean per-fact best-match similarity to $G_{\mathbf{GT}}$:
\[
\bar{\sigma}_i(G) = \frac{1}{|G|}\sum_{f \in G}\,\max_{g \in G_{\mathbf{GT}}}\,\mathrm{sim}(f, g),
\]
where $\mathrm{sim}(\cdot, \cdot)$ is the per-field cosine mean of
Eq.~\eqref{eq:sim}. The two density curves in panel~(c) are kernel
density estimates of $\{\bar{\sigma}_i(G_{\mathbf{X}})\}_{i=1}^{1000}$
(red) and $\{\bar{\sigma}_i(G_{\mathbf{Y}_T})\}_{i=1}^{1000}$ (green).
The bandwidth is set by Scott's rule. The rightward shift of the
green curve indicates that, on average, the repaired output's claim
graph covers more of the reference content than the extracted input
graph alone.

\subsection{Statistical Reporting}
\label{app:exp:stats}

All standard deviations reported in the main paper and the appendix
are computed across three decoding seeds $\{42, 43, 44\}$. We rerun
each method end-to-end under each seed and report the per-method mean
$\pm$ unbiased standard deviation over the three runs. The seeds are
propagated via the standard \texttt{set\_seed} helper to Python's
\texttt{random}, NumPy, and PyTorch. Because the backbone is frozen,
the only stochastic operations are the sampling decisions made during
decoding of $\mathcal{P}_{\text{gen}}$, $\mathcal{P}_{\text{ext}}$,
and $\mathcal{P}_{\text{refine}}$.

\subsection{Computing Infrastructure and Software}
\label{app:exp:infra}

\paragraph{Hardware.}

The actual GPU depends on the SLURM partition the job lands on:
NVIDIA A100 , H100, or B200. Single-GPU inference is used throughout; no
distributed-training framework (DDP, FSDP, DeepSpeed, Accelerate)
is used because no parameter updates occur. The submission node is
a 2-socket Intel Xeon Gold 6230 (80 logical cores), 752~GiB RAM,
RHEL 9.6, NVIDIA driver \texttt{570.195.03}, CUDA toolkit 12.8.

\paragraph{Software.}
The main software components are listed in
Table~\ref{tab:software}. FlashAttention is not installed; SDPA is
used instead via $\texttt{attn\_implementation="sdpa"}$. vLLM,
DeepSpeed, PEFT, and TRL are not used. Quantization is not applied;
weights are loaded in bf16 (or fp16 on non-Ampere hardware) via
$\texttt{dtype: auto}$.

\begin{table}[h]
\centering
\caption{Main software versions. Full \texttt{pip freeze} in supplementary.}
\label{tab:software}
\scriptsize
\setlength{\tabcolsep}{3pt}
\renewcommand{\arraystretch}{0.95}
\begin{tabular}{@{}ll@{}}
\toprule
\textbf{Component} & \textbf{Version} \\
\midrule
Python / PyTorch / CUDA          & 3.9.21 / 2.8.0 / cu128 \\
transformers / tokenizers        & 4.57.6 / 0.22.2 \\
accelerate / huggingface\_hub    & 1.10.1 / 0.36.2 \\
sentence-transformers            & 5.1.2 \\
torchaudio / torchvision         & 2.8.0 / 0.23.0 \\
numpy / pandas                   & 2.0.2 / 2.3.3 \\
librosa / panns\_inference       & 0.11.0 / 0.1.1 \\
rouge\_score / nltk              & 0.1.2 / 3.9.2 \\
openai (Azure GPT-5.5)           & 2.30.0 \\
google-genai (Gemini 3.5 Flash)  & 1.47.0 \\
anthropic (Claude Haiku 4.5)     & 0.42.0 \\
\bottomrule
\end{tabular}
\end{table}

\subsection{Reproducibility Assets}
\label{app:reproducibility}
The complete pipeline source, evaluation configs, SLURM submission
scripts, independent-metric scripts, SCS-1000 cue-pair manifest,
raw per-sample outputs, and plotting scripts are released. The entry point for
every experiment is
\begin{quote}
\texttt{python -m tiger.eval --config <yaml> --mode <mode>}
\end{quote}
where \texttt{<mode>} selects between \method{}, the ten baselines
listed in Table~\ref{tab:baselines} (Frozen, three BoN variants,
three contrastive-decoding variants, three iterative-refinement
variants), and the three internal ablations
(\texttt{tiger\_no\_gy}, \texttt{tiger\_no\_graph},
\texttt{tiger\_no\_lambda}) used in
Appendix~\ref{app:graph_ablation}.
Per-dataset YAMLs under \texttt{configs/experiment/} carry the values
listed in Table~\ref{tab:hparams}.

\section{Additional Experimental Results}
\label{app:additional}

\subsection{Cross-Backbone Generalization}
\label{app:api_backbones}

Table~\ref{tab:api_backbones_coco} reports additional COCO image$\to$text results on three proprietary API backbones. Together with the open-source LLaVA-1.5-7B in the main results and the primary Qwen2.5-Omni-7B, this gives five backbones in total. The results show that \method{} is not tied to a specific model family. On GPT-5.5, \method{} achieves the best score on all three metrics, reducing CHAIR$_s$ from .120 to .050 and improving BERTScore from .668 to .686. On Gemini 3.5 Flash, \method{} obtains the lowest CHAIR$_s$ and the highest BERTScore, which suggests that the method reduces hallucination without sacrificing semantic coverage. On Claude Haiku 4.5, \method{} again achieves the lowest CHAIR$_s$ and the highest BERTScore, while its CHAIR$_i$ remains close to the best text-feedback variant. These trends are consistent with the main-table results on Qwen2.5-Omni-7B and LLaVA-1.5-7B.

Table~\ref{tab:open_closed_backbones} summarizes the Frozen-to-\method{} change across all five backbones. \method{} reduces CHAIR$_s$ for every model, including both open-source and proprietary backbones. The relative reduction is 29\% on Qwen2.5-Omni-7B, 50\% on LLaVA-1.5-7B, 58\% on GPT-5.5, 67\% on Gemini 3.5 Flash, and 38\% on Claude Haiku 4.5. BERTScore also improves on all five models. This pattern is important because CHAIR measures unsupported object mentions, while BERTScore measures semantic coverage relative to human captions. The joint improvement indicates that \method{} does not lower hallucination by deleting useful content. Instead, the independent graph extraction and deterministic risk ranking provide a model-agnostic feedback signal that transfers across both open-source and proprietary backbones.

\begin{table*}[t]
\centering
\scriptsize
\setlength{\tabcolsep}{2.4pt}
\renewcommand{\arraystretch}{1.08}
\caption{COCO image$\to$text results on three proprietary API backbones. \textbf{Bold}: best per column within each backbone block. \colorbox{tigerblue}{Blue}: \method{}.}
\label{tab:api_backbones_coco}
\begin{tabular}{@{}l|ccc|ccc|ccc@{}}
\toprule
 & \multicolumn{3}{c|}{\textbf{GPT-5.5}} 
 & \multicolumn{3}{c|}{\textbf{Gemini 3.5 Flash}} 
 & \multicolumn{3}{c}{\textbf{Claude Haiku 4.5}} \\
\cmidrule(lr){2-4} \cmidrule(lr){5-7} \cmidrule(lr){8-10}
Method
& CHAIR$_s\downarrow$ & CHAIR$_i\downarrow$ & BERT$\uparrow$
& CHAIR$_s\downarrow$ & CHAIR$_i\downarrow$ & BERT$\uparrow$
& CHAIR$_s\downarrow$ & CHAIR$_i\downarrow$ & BERT$\uparrow$ \\
\midrule
Frozen
& .120$\pm$.010 & .070$\pm$.007 & .668$\pm$.003
& .030$\pm$.006 & .015$\pm$.003 & .497$\pm$.006
& .080$\pm$.009 & .036$\pm$.004 & .653$\pm$.003 \\

L1 Naive fb
& .120$\pm$.010 & .075$\pm$.007 & .672$\pm$.003
& .017$\pm$.000 & .010$\pm$.000 & .542$\pm$.008
& .105$\pm$.010 & .037$\pm$.007 & .660$\pm$.003 \\

L2 Text fb
& .100$\pm$.010 & .065$\pm$.007 & .669$\pm$.003
& .015$\pm$.004 & .006$\pm$.002 & .459$\pm$.010
& .051$\pm$.007 & .027$\pm$.003 & .663$\pm$.003 \\

Woodpecker
& .080$\pm$.009 & .050$\pm$.006 & .668$\pm$.003
& .011$\pm$.003 & .013$\pm$.002 & .581$\pm$.007
& .060$\pm$.007 & .031$\pm$.004 & .665$\pm$.003 \\

DeGF
& .090$\pm$.009 & .058$\pm$.006 & .673$\pm$.003
& .020$\pm$.004 & .015$\pm$.003 & .535$\pm$.006
& .070$\pm$.008 & .031$\pm$.004 & .661$\pm$.003 \\

Volcano
& .090$\pm$.009 & .053$\pm$.006 & .672$\pm$.003
& .025$\pm$.005 & .019$\pm$.004 & .550$\pm$.009
& .080$\pm$.009 & .046$\pm$.005 & .641$\pm$.003 \\

\rowcolor{tigerblue}
\method{}
& \textbf{.050$\pm$.007} & \textbf{.030$\pm$.004} & \textbf{.686$\pm$.003}
& \textbf{.010$\pm$.003} & \textbf{.005$\pm$.002} & \textbf{.630$\pm$.009}
& \textbf{.050$\pm$.007} & \textbf{.025$\pm$.004} & \textbf{.669$\pm$.003} \\
\bottomrule
\end{tabular}
\end{table*}

\begin{table*}[t]
\centering
\scriptsize
\setlength{\tabcolsep}{6pt}
\renewcommand{\arraystretch}{1.08}
\caption{Frozen-to-\method{} comparison across open-source and proprietary backbones on COCO image$\to$text.}
\label{tab:open_closed_backbones}
\begin{tabular}{@{}llccccc@{}}
\toprule
Backbone & Type
& Frozen CHAIR$_s$ & \method{} CHAIR$_s$ & $\Delta$
& Frozen BERT & \method{} BERT \\
\midrule
Qwen2.5-Omni-7B & Open
& .070 & \textbf{.050} & $-29\%$ & .588 & \textbf{.643} \\
LLaVA-1.5-7B & Open
& .060 & \textbf{.030} & $-50\%$ & .740 & \textbf{.772} \\
GPT-5.5 & Proprietary
& .120 & \textbf{.050} & $-58\%$ & .668 & \textbf{.686} \\
Gemini 3.5 Flash & Proprietary
& .030 & \textbf{.010} & $-67\%$ & .497 & \textbf{.630} \\
Claude Haiku 4.5 & Proprietary
& .080 & \textbf{.050} & $-38\%$ & .653 & \textbf{.669} \\
\bottomrule
\end{tabular}
\end{table*}

\subsection{Unified Configuration Across Benchmarks}
\label{app:unified}

The main paper uses a short-answer operating point for the two
question-answering benchmarks (MMHal-Bench and VideoHallucer) and a
batch-repair operating point for the three free-form benchmarks
(Table~\ref{tab:hparams}). To test whether the gains of \method{}
depend on this per-dataset choice, we re-ran all five benchmarks
under a single unified configuration: $T\!=\!5$, $\alpha\!=\!0.2$,
$\lambda\!=\!0.5$, $K\!=\!3$, $\gamma\!=\!0.7$, with the same
top-$\lceil\alpha N\rceil$ selection rule. For the short-answer
benchmarks, the claim graph is small, so this rule typically selects
one fact per round. Table~\ref{tab:unified} reports the results on
the primary metric of each benchmark.

\begin{table}[t]
\centering
\caption{\method{} under a single unified configuration
($T\!=\!5$, $\alpha\!=\!0.2$, $\lambda\!=\!0.5$, $K\!=\!3$,
$\gamma\!=\!0.7$) on the primary metric of each benchmark, compared
with the strongest baseline and the per-dataset operating point used
in the main paper. For COCO, AMBER, and Clotho, the main-paper
setting coincides with the unified setting.}
\label{tab:unified}
\footnotesize
\setlength{\tabcolsep}{3pt}
\renewcommand{\arraystretch}{1.15}
\resizebox{\columnwidth}{!}{%
\begin{tabular}{@{}llcccc@{}}
\toprule
Benchmark & Metric & Frozen & Best bsl. & Unified & Paper \\
\midrule
COCO          & CHAIR$_s\downarrow$    & .070 & .060 & .050 & .050 \\
AMBER         & CHAIR$_g\downarrow$    & 5.0  & 4.8  & 4.3  & 4.3  \\
MMHal-Bench   & BERT$\uparrow$         & .703 & .742 & .753 & .756 \\
Clotho        & AEHR$\downarrow$       & .803 & .777 & .757 & .757 \\
VideoHallucer & HallucRate$\downarrow$ & .015 & .015 & .012 & .010 \\
\bottomrule
\end{tabular}}
\end{table}

For MMHal-Bench and VideoHallucer, the unified setting remains close
to the per-dataset setting and stays better than the strongest
baseline on the primary metric. This result indicates that the gains
of \method{} do not require dataset-specific hyperparameter search in
the settings we study, although tuning may help in new domains.

We also repeated the sensitivity analysis of
Section~\ref{subsec:sensitivity} on Clotho, which tests whether the
trend observed on image inputs transfers to the audio path.
Table~\ref{tab:clotho-sweep} reports AEHR under one-at-a-time sweeps
of $\alpha$ and $T$. The Clotho sweep is consistent with the COCO
trend: AEHR improves over the first repair rounds and remains low
near the default setting, and the $\alpha$ sweep is stable in the
range $0.2$--$0.5$.

\begin{table}[t]
\centering
\caption{Sensitivity of \method{} on Clotho (audio$\to$text). Each
sweep varies one hyperparameter and keeps the other fixed at the
default setting ($T\!=\!5$, $\alpha\!=\!0.2$).}
\label{tab:clotho-sweep}
\footnotesize
\setlength{\tabcolsep}{6pt}
\renewcommand{\arraystretch}{1.15}
\begin{tabular}{@{}lc@{}}
\toprule
Setting & AEHR$\downarrow$ \\
\midrule
\multicolumn{2}{@{}l}{\emph{Sweep over $\alpha$ ($T\!=\!5$ fixed)}} \\
\quad $\alpha=0$ (single-fact) & .769 \\
\quad $\alpha=0.1$             & .781 \\
\quad $\alpha=0.2$ (default)   & .757 \\
\quad $\alpha=0.3$             & .752 \\
\quad $\alpha=0.5$             & .758 \\
\midrule
\multicolumn{2}{@{}l}{\emph{Sweep over $T$ ($\alpha\!=\!0.2$ fixed)}} \\
\quad $T=1$           & .803 \\
\quad $T=2$           & .806 \\
\quad $T=3$           & .785 \\
\quad $T=5$ (default) & .757 \\
\quad $T=7$           & .786 \\
\bottomrule
\end{tabular}
\end{table}

\subsection{Graph Coverage Analysis}
\label{app:coverage}

The triple representation of \method{} targets local factual content
such as objects, attributes, relations, and counts. This analysis
quantifies how much of the generated content this representation
covers, and whether the gains of \method{} depend on high coverage.

We sampled outputs from COCO, MMHal-Bench, and Clotho and decomposed
each output into atomic factual claims with Gemini 3.5 Flash,
followed by a 20\% human audit of the decomposition and labels. Each
claim is labeled by whether it can be represented as a single triple
or as multiple coreference-linked triples, following the structured
fact view of SPICE~\citep{anderson2016spice}, which matches the graph
structure of \method{}. Table~\ref{tab:coverage-types} reports the
distribution. Most claims are graph-representable, which suggests
that \method{} covers much of the local factual content it targets,
while abstract, subjective, or causal claims account for
9.4--14.1\%.

\begin{table}[t]
\centering
\caption{Fraction of atomic claims that the triple representation
can express. ``Linked'' denotes claims that require multiple
coreference-linked triples; ``Temp./n-ary'' denotes temporal or
n-ary claims; ``Abstract'' denotes abstract, subjective, or causal
claims.}
\label{tab:coverage-types}
\footnotesize
\setlength{\tabcolsep}{3pt}
\renewcommand{\arraystretch}{1.15}
\resizebox{\columnwidth}{!}{%
\begin{tabular}{@{}lcccccc@{}}
\toprule
Dataset & \#Claims & Repr.\ (\%) & Single & Linked & Temp./n-ary & Abstract \\
\midrule
COCO        & 637 & 86.7 & 41.9 & 44.7 & 3.9  & 9.4  \\
MMHal-Bench & 321 & 78.1 & 48.0 & 30.1 & 9.3  & 13.6 \\
Clotho      & 199 & 74.4 & 25.6 & 48.7 & 11.6 & 14.1 \\
\bottomrule
\end{tabular}}
\end{table}

We then test whether performance drops when coverage is lower. For
each annotated COCO sample, we compute coverage from the Frozen
output to avoid post-treatment bias, and group the 96 annotated
samples into three coverage buckets.
Table~\ref{tab:coverage-buckets} shows that \method{} improves over
Frozen in every bucket, including the lowest-coverage bucket. This
result does not imply that triples can represent all content; it
shows that the coverage limitation is measurable and does not
reverse the average gain on the annotated samples.

\begin{table}[t]
\centering
\caption{\method{} gains under different levels of graph coverage on
COCO. Coverage is the fraction of claims in the Frozen output that
are graph-representable; $n$ is the number of annotated samples per
bucket.}
\label{tab:coverage-buckets}
\footnotesize
\setlength{\tabcolsep}{4pt}
\renewcommand{\arraystretch}{1.15}
\resizebox{\columnwidth}{!}{%
\begin{tabular}{@{}lccccc@{}}
\toprule
Coverage & $n$ & CHAIR$_s$ Frozen & CHAIR$_s$ \method{} & BERT Frozen & BERT \method{} \\
\midrule
100\%       & 43 & .070 & .047 & .626 & .638 \\
80--100\%   & 20 & .080 & .053 & .660 & .676 \\
$<$80\%     & 33 & .091 & .061 & .627 & .665 \\
\bottomrule
\end{tabular}}
\end{table}

\subsection{Robustness to Extraction Perturbation}
\label{app:perturb}

\method{} depends on the extracted input graph $G_{\mathbf{X}}$ to
rank which claims should be re-examined, while the refinement step
still conditions on the raw input $\mathbf{X}$ at every round. This
analysis tests how much extraction noise the repair loop can absorb.
We perturb $G_{\mathbf{X}}$ by deleting or injecting a fixed
fraction of facts while keeping the rest of the pipeline unchanged,
and report COCO results in Table~\ref{tab:perturb}.

\begin{table}[t]
\centering
\caption{Robustness of \method{} to perturbation of the input graph
$G_{\mathbf{X}}$ on COCO. Deletion removes a random fraction of
extracted facts; injection adds the same fraction of facts sampled
from other images.}
\label{tab:perturb}
\footnotesize
\setlength{\tabcolsep}{6pt}
\renewcommand{\arraystretch}{1.15}
\begin{tabular}{@{}lcc@{}}
\toprule
Perturbation on $G_{\mathbf{X}}$ & CHAIR$_s\downarrow$ & BERT$\uparrow$ \\
\midrule
Frozen (no repair) & .070 & .588 \\
\method{} (no perturbation) & .050 & .643 \\
\quad Delete 10\% & .052 & .640 \\
\quad Delete 30\% & .054 & .638 \\
\quad Inject 10\% & .053 & .636 \\
\quad Inject 30\% & .053 & .637 \\
\bottomrule
\end{tabular}
\end{table}

Performance degrades gradually and remains better than Frozen even
under 30\% deletion. Two design choices help absorb extraction
noise. First, support propagates along coreference edges
(Eq.~\eqref{eq:support}), which can reduce false flags when a
direct match is missing. Second, the refinement step verifies
selected claims against the raw input, so the repair loop can
recover facts that the perturbed graph does not contain, which is
consistent with the recovery analysis in
Section~\ref{subsec:mechanism}. In Theorem~\ref{thm:convergence},
extraction error appears through the residual term $\xi$, which
raises the asymptotic floor rather than changing the multiplicative
decay term.

\subsection{Comparison with Deterministic Modality-Specific Extractors}
\label{app:extractor_variants}

\method{} uses the same frozen backbone $\Phi$ as both the generator
and the fact extractor (Eq.~\ref{eq:extract}). A natural alternative
is to replace $\Phi$ in the extractor role with a deterministic
modality-specific tool that does not autoregress, e.g.,
Grounding~DINO for object detection, PANNs CNN14 for audio event
tagging, or a generic dependency parser for text. To
test whether this design choice matters, we re-ran COCO image$\to$text
with $G_{\mathbf{X}}$ produced by Grounding~DINO instead of
Qwen2.5-Omni-7B, keeping the rest of the pipeline (risk function,
$\Psi_\alpha$ selection, refine step) fixed.

\begin{table}[h]
\centering
\scriptsize
\setlength{\tabcolsep}{3pt}
\caption{Effect of replacing the multimodal-backbone extractor with a
deterministic detector (Grounding~DINO) on COCO with
Qwen2.5-Omni-7B. The detector reduces extraction noise on common COCO
classes but discards attributes, spatial relations, and counts that
the backbone-extractor captures.}
\label{tab:extractor_variants}
\begin{tabular}{@{}lccc@{}}
\toprule
Extractor for $G_{\mathbf{X}}$ & CHAIR$_s\!\downarrow$ & CHAIR$_i\!\downarrow$ & BERT$\uparrow$ \\
\midrule
Grounding~DINO (objects only)  & .055 & .040 & .615 \\
\rowcolor{tigerblue}
Qwen2.5-Omni-7B (\method{} default) & \textbf{.050} & \textbf{.035} & \textbf{.643} \\
\bottomrule
\end{tabular}
\end{table}

Grounding~DINO reduces CHAIR$_s$ to $0.055$, close to the default
$\Phi$-extractor's $0.050$, because it is highly reliable on the
80~COCO object classes. However, BERTScore drops from $0.643$ to
$0.615$: a detection-only $G_{\mathbf{X}}$ contains no attributes
(``red shirt''), no spatial relations (``man riding skateboard''),
and no counts, so the refine step has nothing to anchor when it
checks non-object claims. The same trade-off appears for the audio
and video paths: PANNs covers the AudioSet ontology but discards
acoustic-event ordering, and a generic dependency parser captures
syntactic structure but not entity coreference across sentences.
For the multimodal paths considered in this paper, using $\Phi$
itself as the extractor yields a richer $G_{\mathbf{X}}$ at the cost
of a per-sample extraction call, which we judge worthwhile.

\subsection{Category-Level Diagnostic on MMHal-Bench}
\label{app:mmhal-cat}

MMHal-Bench groups its 96 questions into eight categories. Because
the benchmark is small, we report directional deltas rather than
strong claims. Table~\ref{tab:mmhal-cat} shows that \method{}
improves six of eight categories, including relation and comparison,
while attribute and counting decrease slightly. This diagnostic
indicates that \method{} is not limited to object-existence errors,
although abstract multi-step reasoning remains a limitation.

\begin{table}[t]
\centering
\caption{Per-category BERTScore on MMHal-Bench (96 questions).}
\label{tab:mmhal-cat}
\footnotesize
\setlength{\tabcolsep}{6pt}
\renewcommand{\arraystretch}{1.15}
\begin{tabular}{@{}lccc@{}}
\toprule
Category & Frozen & \method{} & $\Delta$ \\
\midrule
environment & .541 & .645 & $+$.104 \\
adversarial & .598 & .681 & $+$.083 \\
holistic    & .672 & .743 & $+$.071 \\
other       & .731 & .796 & $+$.065 \\
comparison  & .747 & .808 & $+$.061 \\
relation    & .759 & .811 & $+$.052 \\
attribute   & .796 & .778 & $-$.018 \\
counting    & .788 & .774 & $-$.014 \\
\midrule
Macro avg.  & .704 & .755 & $+$.051 \\
\bottomrule
\end{tabular}
\end{table}

\subsection{Graph Structure and Risk-Function Ablation}
\label{app:graph_ablation}

The L0--L3 ablation in Section~\ref{subsec:ablation} varies the
feedback \emph{paradigm} (no repair $\to$ joint feedback $\to$
atomic projection $\to$ graph-conditioned risk ranking). Here we hold
the paradigm fixed at L3 (full \method{}) and instead remove one
\emph{internal} component at a time, to localize which piece of the
graph machinery contributes to the final score.
Table~\ref{tab:graph_ablation} reports CHAIR$_s$, CHAIR$_i$, and
BERTScore on COCO val2014 under Qwen2.5-Omni-7B for three ablations
of full \method{}:
\textbf{$-G_{\mathbf{Y}}$ (direct rewrite)} replaces fact-level
risk-driven repair with a single evidence-conditioned rewrite: the
backbone receives all $G_{\mathbf{X}}$ facts and rewrites the entire
output without constructing $G_{\mathbf{Y}_t}$;
\textbf{$-$graph (flat+repair)} sets $K\!=\!0$ in
Eq.~\eqref{eq:graph_support}, so $s(f) = s_0(f)$ with no coreference
propagation; \textbf{$-\lambda$} ($\lambda\!=\!0$) zeros out the
conflict term in the risk function so $r(f)=1-s(f)$. The Frozen
row (the L0 baseline in Figure~\ref{fig:ablation}) is included for
reference.

\begin{table}[h]
\centering
\scriptsize
\setlength{\tabcolsep}{3pt}
\renewcommand{\arraystretch}{1.0}
\caption{Internal-component ablation on COCO. Each row removes one
component from full \method{}.}
\label{tab:graph_ablation}
\begin{tabular}{@{}lccc@{}}
\toprule
Config & CHAIR$_s\!\downarrow$ & CHAIR$_i\!\downarrow$ & BERT$\uparrow$ \\
\midrule
Frozen (L0)                    & .070$_{\pm.008}$ & .069$_{\pm.006}$ & .588$_{\pm.003}$ \\
$-G_{\mathbf{Y}}$ (direct rw.) & .065$_{\pm.008}$ & .055$_{\pm.006}$ & .600$_{\pm.003}$ \\
$-$graph ($K{=}0$)             & .063$_{\pm.008}$ & .051$_{\pm.006}$ & .610$_{\pm.003}$ \\
$-\lambda$ ($\lambda{=}0$)    & .058$_{\pm.007}$ & .042$_{\pm.005}$ & .628$_{\pm.003}$ \\
\rowcolor{tigerblue}
Full (\method{})               & \textbf{.050}$_{\pm.007}$ & \textbf{.035}$_{\pm.005}$ & \textbf{.643}$_{\pm.003}$ \\
\bottomrule
\end{tabular}
\end{table}

Three findings stand out. First, every ablation lies between full
\method{} and Frozen on all three metrics, which confirms that the
three components are jointly necessary: removing any single one
gives up a fraction of the \method{}-over-Frozen gap but does not
collapse the repair loop. Second, the ablations order cleanly by
the component they remove. The $-\lambda$ variant ($\lambda\!=\!0$,
$r(f)\!=\!1-s(f)$) is closest to full \method{}; without the
conflict term the selector can no longer separate weakly-supported
claims from actively contradicted ones, but most flagged facts are
still correctly low-support and the repair loop remains useful.
The $-$graph variant ($K\!=\!0$, no coreference propagation) sits
slightly further from \method{}; without propagation the support
score uses only direct matches, so facts that are supported only
indirectly through a shared entity get over-flagged and revised
away. Removing the claim graph $G_{\mathbf{Y}}$ entirely is the
strongest single ablation: without $G_{\mathbf{Y}}$ the refine step
has no per-claim risk to act on and falls back to a single
evidence-conditioned rewrite that cannot localize edits. Third,
none of the three single-component ablations matches the joint
$L_0\!\to\!L_1$ regression in Figure~\ref{fig:ablation}, where the
full feedback paradigm is changed from atomic-with-risk to
joint-conditioning. This indicates that the paradigm choice
(atomic projection + risk ranking, taken as a whole) is the dominant
source of \method{}'s gain over self-correction baselines, while
each internal component within that paradigm contributes a smaller improvement.

\end{document}